\documentclass[letterpaper]{article} 
\usepackage{aaai2026}  
\usepackage{times}  
\usepackage{helvet}  
\usepackage{courier}  
\usepackage[hyphens]{url}  
\usepackage{graphicx} 
\usepackage{natbib}  
\usepackage{caption} 
\usepackage{algorithm}
\usepackage{algorithmic}

\usepackage{booktabs}       
\usepackage{amsfonts}       
\usepackage{nicefrac}       
\usepackage{subfigure}
\usepackage{bm}
\usepackage{bbold}
\usepackage{amsmath}
\usepackage{amssymb}
\usepackage{mathtools}
\usepackage{amsthm}

\theoremstyle{plain}
\newtheorem{theorem}{Theorem} 

\newtheorem{lemma}[theorem]{Lemma}

\theoremstyle{definition}

\theoremstyle{remark}

\newcommand{\norm}[1]{\left\lVert#1\right\rVert}

\usepackage{xspace}
\usepackage{cancel}

\usepackage{pythonhighlight}
\usepackage{multirow}
\newcommand{\STAB}[1]{\begin{tabular}{@{}c@{}}#1\end{tabular}}

\usepackage{xpatch}
\xpatchcmd{\boxed}{%
\fbox}{%
\fcolorbox{red}{white}}{}{}
\makeatletter
\DeclareRobustCommand\onedot{\futurelet\@let@token\@onedot}
\def\@onedot{\ifx\@let@token.\else.\null\fi\xspace}

\let\titleold\title
\renewcommand{\title}[1]{\titleold{#1}\newcommand{\thetitle}{#1}}
\def\maketitlesupplementary
   {
   \newpage
       \twocolumn[
        \centering
        \Large
        \textbf{\thetitle}\\
        \vspace{0.5em}Appendix \\
        \vspace{1.0em}
       ] 
   }

\makeatother

\usepackage{xcolor}

\usepackage{newfloat}
\usepackage{listings}
\DeclareCaptionStyle{ruled}{labelfont=normalfont,labelsep=colon,strut=off} 
\floatstyle{ruled}
\newfloat{listing}{tb}{lst}{}
\floatname{listing}{Listing}
\title{Integrating Reweighted Least Squares with Plug-and-Play Diffusion Priors for Noisy Image Restoration}
\author{
    Ji Li\textsuperscript{\rm 1},
    Chao Wang\textsuperscript{\rm 2}\thanks{Corresponding author.}
}
\affiliations{
    \textsuperscript{\rm 1}Academy for Multidisciplinary Studies, Capital Normal University, Beijing, China\\
    \textsuperscript{\rm 2}College of Science, China Agricultural University, Beijing, China\\
    matliji@163.com, wywwwnx@163.com
}

\begin{document}

\maketitle

\begin{abstract}
Existing plug-and-play image restoration methods typically employ off-the-shelf Gaussian denoisers as proximal operators within classical optimization frameworks based on variable splitting. Recently, denoisers induced by generative priors have been successfully integrated into regularized optimization methods for image restoration under Gaussian noise. However, their application to non-Gaussian noise—such as impulse noise—remains largely unexplored. In this paper, we propose a plug-and-play image restoration framework based on generative diffusion priors for robust removal of general noise types, including impulse noise. Within the maximum a posteriori (MAP) estimation framework, the data fidelity term is adapted to the specific noise model. Departing from the conventional least-squares loss used for Gaussian noise, we introduce a generalized Gaussian scale mixture-based loss, which approximates a wide range of noise distributions and leads to an $\ell_q$-norm ($0<q\leq2$) fidelity term. This optimization problem is addressed using an iteratively reweighted least squares (IRLS) approach, wherein the proximal step involving the generative prior is efficiently performed via a diffusion-based denoiser. Experimental results on benchmark datasets demonstrate that the proposed method effectively removes non-Gaussian impulse noise and achieves superior restoration performance.
\end{abstract}

\setlength{\abovedisplayskip}{3pt}
\setlength{\belowdisplayskip}{3pt}


\section{Introduction}

Image restoration, encompassing tasks such as denoising, deblurring, super-resolution and inpainting, is an important and longstanding problem in computer vision and signal processing. It is typically formulated as an inverse problem, aiming to recover the latent clean image $\bm{x}$ from a corrupted observation $\bm{y}$, modeled by the following system:
\begin{equation}
  \label{eq:prob}
  \bm{y} = \mathcal{D}(\bm{A}\bm{x}),
\end{equation}
where $\bm{x}\in\mathbb{R}^n$ denotes the clean image, $\bm{A}\in\mathbb{R}^{m\times n}$ models the degradation operator, $\mathcal{D}:\mathbb{R}^m\to\mathbb{R}^m$ represents the noise corruption process. The system~\eqref{eq:prob} is generally ill-posed, as the dimension of measurement $m$ is often less than that of the latent image $n$.

To stabilize the recovery of $\bm{x}$, variational optimization approaches introduce regularization that incorporates prior knowledge about the solution. A common formulation is:
\begin{equation}
  \label{eq:2}
  \min_{\bm{x}}\quad \mathcal{S}(\bm{A}\bm{x},\bm{y}) + \lambda \mathcal{R}(\bm{x}),
\end{equation}
where $\mathcal{S}(\bm{A}\bm{x},\bm{y})$ measures the fidelity to the observed data, $\mathcal{R}(\bm{x})$ enforces the prior information about the statistics of the latent image~\citep{chambolle2010introduction,beck2009fast,chan2006total,rudin1992nonlinear,figueiredo2001wavelet}. Notably, this formulation corresponds to a maximum a posteriori (MAP) estimation of the posterior density $p(\bm{x}|\bm{y})$, where $\mathcal{S}(\bm{A}\bm{x},\bm{y})=-\log p(\bm{y}|\bm{x})$ and $\lambda \mathcal{R}(\bm{x})=-\log p(\bm{x})$ respectively.

The fidelity term $\mathcal{S}(\bm{A}\bm{x}, \bm{y})$ should ideally reflect the noise distribution, while the regularization term $\mathcal{R}(\bm{x})$ should approximate the logarithm of the image prior. The plug-and-play (PnP) framework, particularly regularization-by-denoising (RED)~\citep{reehorst2018regularization,romano2017little}, has emerged as a widely adopted and flexible solution for various image restoration tasks. Within this framework, regularization is implicitly enforced via a denoising operator, which can be realized using classical methods such as BM3D\citep{dabov2007image}, deep neural network-based denoisers~\citep{zhang2017beyond}, or denoisers derived from generative models~\citep{chung2022diffusion, zhu2023denoising, martin2024pnp}.

\begin{figure*}[!htbp]
   \centering 
   \begin{tabular}{c@{\hspace*{2pt}}c@{\hspace*{0pt}}c@{\hspace*{0pt}}c@{\hspace*{0pt}}c@{\hspace*{0pt}}c@{\hspace*{0pt}}c@{\hspace*{0pt}}c@{\hspace*{0pt}}c@{\hspace*{0pt}}c@{\hspace*{0pt}}c@{\hspace*{2pt}}c@{\hspace*{2pt}}c@{\hspace*{0pt}}}
      \raisebox{1.0\height}{\includegraphics[width = 0.07\textwidth]{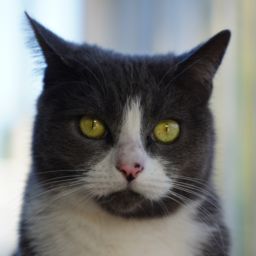}}
     &
       \includegraphics[width = 0.3\textwidth]{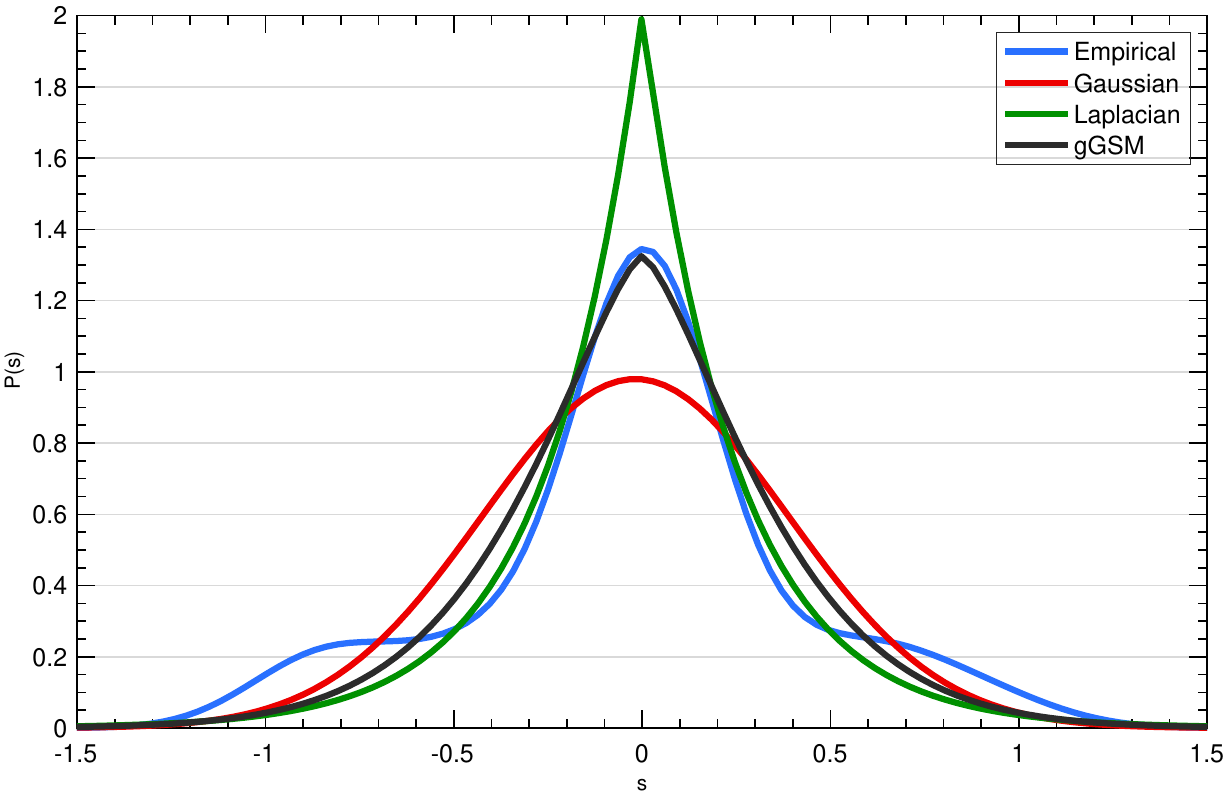}
       \llap{\raisebox{55pt}{
      \includegraphics[width = 0.07\textwidth]{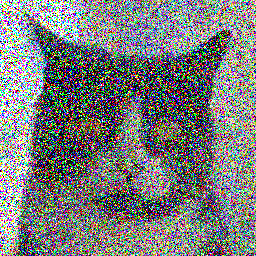}%
       \hspace*{100pt}}}
     & \includegraphics[width = 0.3\textwidth]{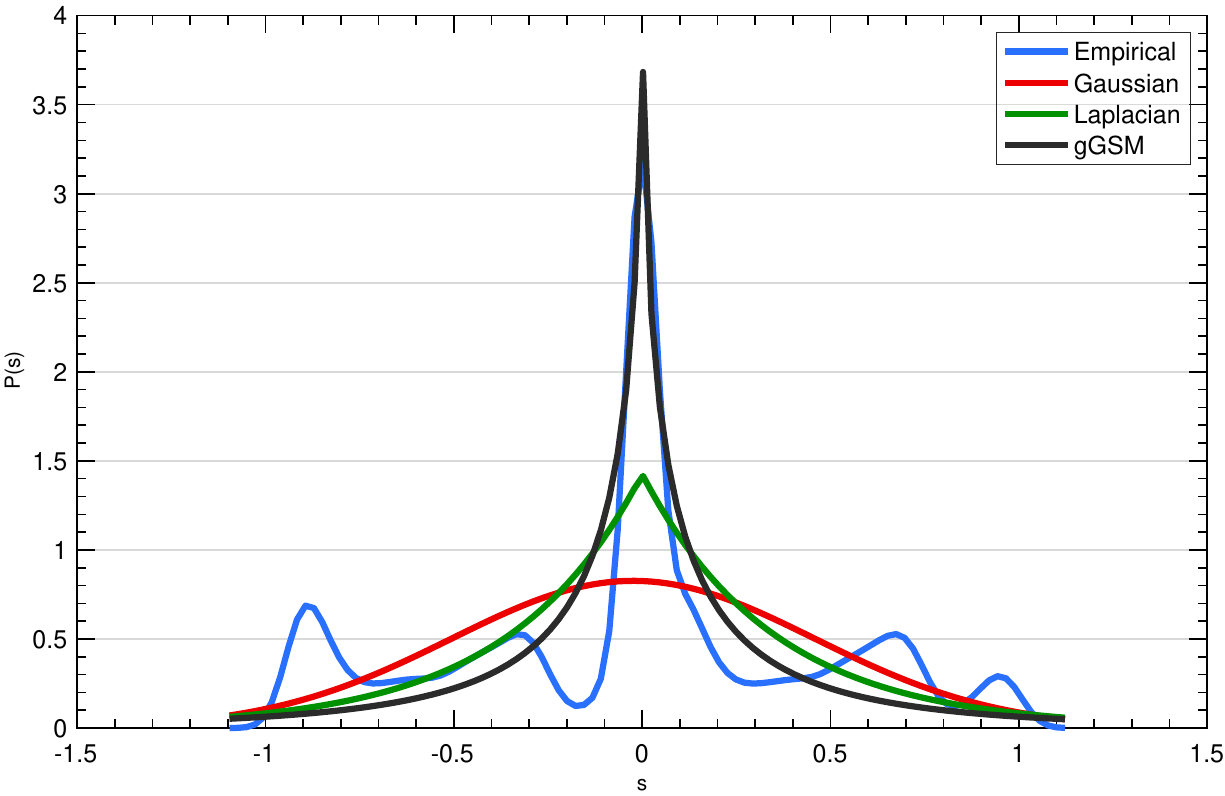}
       \llap{\raisebox{55pt}{
      \includegraphics[width = 0.07\textwidth]{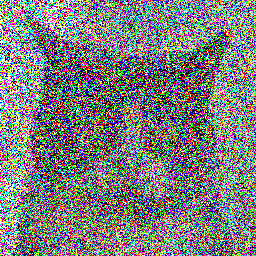}%
       \hspace*{100pt}}}
     & \includegraphics[width = 0.3\textwidth]{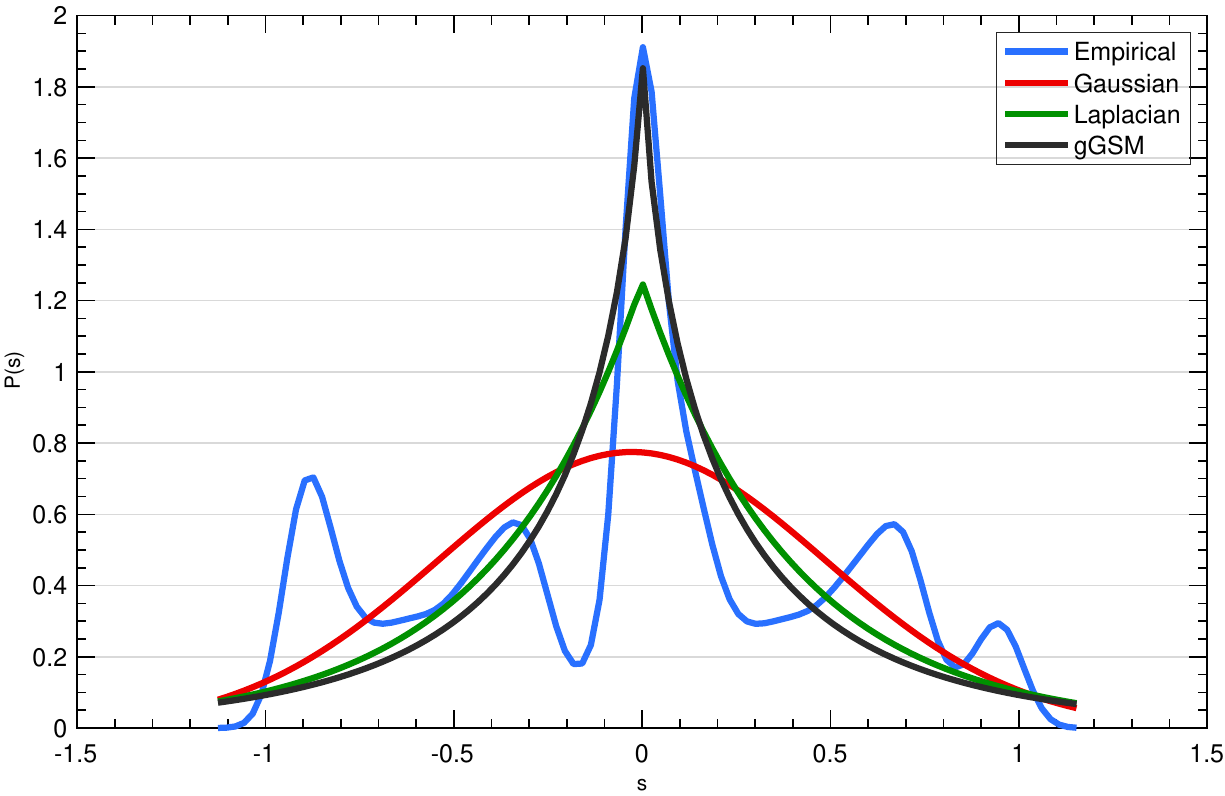}
     \llap{\raisebox{55pt}{
      \includegraphics[width = 0.07\textwidth]{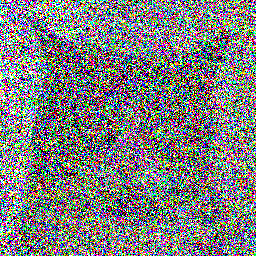}%
       \hspace*{100pt}}}\\[-5pt]
     \tiny{(a)} & \tiny{(b)} & \tiny{(c)} & \tiny{(d)}\vspace*{-8pt}
  \end{tabular}
   \caption{\textbf{Empirical distribution of impulse noise and its fitted approximation densities (including Laplacian, Gaussian and generalized Gaussian scale mixture) using maximum likelihood estimation.} (a) The clean image; (b) Noise level $s_{sp}=0.5$, the fitted parameter $\theta=0.25, \delta=0.44,q = 1.5$. (c) Noise level $s_{sp}=0.7$, the fitted parameter $\theta=0.35, \delta=0.05,q = 0.5$. (d) Noise level $s_{sp}=0.8$, the fitted parameter $\theta=0.40, \delta=0.20,q = 0.7$. For (b)-(d), the noisy image displays on the top left corner. \label{fig:demo}} \vspace*{-10pt}
 \end{figure*}

 Given a noiseless measurement $\bm{y}_{\text{nl}}:=\bm{A}\bm{x}$, the noise corruption process can be expressed as $\bm{y}=\bm{y}_{\text{nl}}+\bm{s}$. When $\bm{s}$ is a pixel-independent Gaussian, the likelihood $p(\bm{y}|\bm{x})$ leads to a least-squares loss in~\eqref{eq:2}. For non-Gaussian, pixel-dependent noise, the exact form of $p(\bm{y}|\bm{x})$ may be unknown. To address this, we approximate the noise distribution using a generalized Gaussian scale mixture (gGSM) model:
\begin{equation}
  \label{eq:3}
    P(s;\delta,q)  = \frac{q}{2\delta\Gamma(1/q)}\exp\left(-\frac{|s|^q}{\delta^q}\right),
\end{equation}
where $\delta,q$ are the distribution parameters, and $\Gamma(x)=\int_0^{\infty} t^{x-1}e^{-t} dt$ is the Gamma function. The gGSM contains two special known cases: $q=1$ corresponds to the Laplacian scale mixture model
\begin{equation}
  P(s;\theta) = \frac{1}{2\theta}\exp\left(-\frac{|s|}{\theta}\right)
\end{equation}
with parameter $\theta$; And $q=2$ yields the Gaussian distribution.

As a motivating example, we consider impulse noise, especially the salt-and-pepper noise~\citep{windyga2001fast}, which frequently arises in data acquisition and transmission due to faulty sensors or sharp and sudden disturbances in analog-to-digital conversion. Salt-and-pepper noise sets random pixels to the minimum or maximum intensity with equal probability, controlled by a noise level $r_{sp}$. Figure~\ref{fig:demo} compares the empirical distribution of impulse noise with several approximation models, including Gaussian, Laplacian, and gGSM. Among them, gGSM provides the best fit for modeling the noise distribution, typically with $q \in (0, 1]$.

Given that the gGSM model leads to an $\ell_q$-norm ($0 < q \leq 2$) data fidelity term in~\eqref{eq:2}, we aim to solve the corresponding minimization problem using a plug-and-play approach with generative priors. Specifically, we propose an efficient iteratively reweighted least squares (IRLS) algorithm to solve a surrogate problem derived from the $\ell_q$-norm formulation. This surrogate is then addressed using a forward-backward splitting method, where the proximal step is performed via a denoiser induced by a generative diffusion model.

 To summarize, our main contributions are as follows:
 \begin{enumerate}
 \item We introduce a plug-and-play image restoration framework for general noise removal, formulating an $\ell_q$-norm minimization problem using the gGSM model to characterize noise distributions;
 \item We propose an IRLS-based majorization-minimization (MM) scheme~\citep{sun2016majorization} to solve the $\ell_q$-norm sub-problem. This is efficiently addressed using forward-backward splitting, with the proximal step handled via a generative diffusion denoiser;
 \item To mitigate potential division-by-zero issues in IRLS, we propose a minor modification over the IRLS scheme with a decreasing hyperparameter, following~\citep{chartrand2008iteratively}.
 \end{enumerate}

\section{Related Works}

This section reviews traditional image restoration methods, plug-and-play approaches with classical and generative priors, and supervised learning-based restoration methods.

\subsection{Traditional Method for Image Restoration}

Traditional image restoration methods are typically distinguished by their choice of data fidelity term and image prior. A widely used class of priors involves sparsity-promoting regularization, such as total variation (TV) and its variants~\citep{blomgren1997total, vogel1996iterative}, or sparsity in wavelet-transformed domains~\citep{figueiredo2001wavelet}. These priors promote smoothness in the restored image, often prioritizing the recovery of low-frequency components. However, due to a mismatch between such handcrafted regularizations and the true image prior, the recovered images are generally suboptimal. The choice of data fidelity term is often noise-model-dependent. The $\ell_2$-norm is commonly used for Gaussian noise, while the $\ell_1$-norm is more suitable for sparse noise. For impulse noise, the $\ell_0$-norm is popular but challenging to optimize; it is often replaced by continuous concave approximations~\citep{koshelev2023iterative, zhang2015reweighted, ochs2015iteratively}.

\subsection{Plug-and-play Methods}

Since introduced in~\citep{venkatakrishnan2013plug}, the idea of plug-and-play has gained significant attention due to its flexibility and effectiveness in addressing diverse image restoration problems. With fixed data fidelity term, the regularized optimization is generally solved by a splitting method. The regularization term is handled implicitly through a denoising operation, which acts as a proximal step. There are many effective denoiser candidates can be used, such as BM3D~\citep{dabov2007image} denoiser and DnCNN/DRUnet denoiser~\citep{zhang2017beyond,zhang2021plug}. The incorporation of the denoiser into regularized optimization has been adopted to solve image restoration~\citep{zhang2020plug,ryu2019plug,venkatakrishnan2013plug,kamilov2017plug}, and medical image reconstruction~\citep{eksioglu2016decoupled,metzler2015bm3d,metzler2016bm3d}. The regularization-by-denoising (RED) is another way to explicitly formulate the regularization~\citep{romano2017little}.

\subsection{Plug-and-play with Generative Prior}

Recently, powerful diffusion generative models have been leveraged as priors for image restoration tasks. Rather than modeling image density directly, diffusion models learn to approximate the score function $\nabla_{\bm{x}}\log p(\bm{x})$~\citep{song2019generative, song2020score}. Among generative models, diffusion models have demonstrated superior performance compared to alternatives such as variational autoencoders (VAEs) and generative adversarial networks (GANs), making them highly appealing for plug-and-play image restoration. In zero-shot learning settings, off-the-shelf diffusion models can be used for observation-guided conditional sampling from the posterior $p(\bm{x}|\bm{y})$. A key step involves approximating the conditional score $\nabla_{\bm{x}_t}\log p(\bm{x}_t|\bm{y})$ using the identity:
\begin{equation*}
  \nabla_{\bm{x}_t}\log p(\bm{x}_t|\bm{y}) = \nabla_{\bm{x}_t}\log p(\bm{x}_t) + \nabla_{\bm{x}_t}\log p(\bm{y}|\bm{x}_t).
\end{equation*}
Thus, estimating $\nabla_{\bm{x}_t}\log p(\bm{y}|\bm{x}_t)$ becomes critical. \citet{chung2022diffusion} leveraged estimate $p(\bm{y}|\bm{x}_t)\simeq p(\bm{y}|\mathbb{E}[\bm{x}_0|\bm{x}_t])$, with $\mathbb{E}[\bm{x}_0|\bm{x}_t]$ estimated by Tweedie's formula. This approach requires backpropagation through the diffusion model, which is computationally intensive. Moreover, their approximation has been considered coarse. \citet{song2022pseudoinverse} proposed $\Pi$GDM method with a Gaussian estimate of unknown $p(\bm{x}_0|\bm{x}_t)$. However, this method is limited to linear degradation models with Gaussian noise and does not extend naturally to non-Gaussian noise. 

An alternative approach treats image restoration not as conditional generation but as optimization. DiffPIR~\citep{zhu2023denoising} applies half-quadratic splitting (HQS) to~\eqref{eq:2}, using a diffusion model as a denoiser, thereby avoiding backpropagation through the model. However, their method requires multiple iterations to solve each subproblem. Along this line, PnP-Flow~\citep{martin2024pnp} replaces HQS with forward-backward splitting~\citep{parikh2014proximal}, incorporating flow-based generative models instead of diffusion models. Additionally, \citet{mardanivariational} derived an explicit regularizer for diffusion priors, but the computation of its gradient is prohibitively expensive.

\subsection{Supervised Learning for Image Restoration}

The supervised learning solutions for image restoration typically train end-to-end convolutional neural networks using paired datasets~\citep{gilton2019neumann,meinhardt2017learning,jin2017deep}. For linear image restoration, the optimization unrolling scheme helps design an efficient network architecture, which enables better interpretation and training datasets reduction~\citep{adler2018learned,aggarwal2018modl}. These methods require retraining when the task changes, hence lacking flexibility for solving a series of tasks, even when the images share the same density.

\section{Methodology}

Before presenting our method, we give a brief introduction of the diffusion generative model, the formulated diffusion model regularized optimization and the majorization-minimization algorithm for $\ell_q$-norm minimization.

\subsection{Regularization via Diffusion Generative Model}

The diffusion model comprises of two processes: forward process and generative process. The forward process progressively corrupts clean data $\bm{x}_0$ using a sequence of pre-configured noise scales. The seminar work~\citep{song2020score} provided a continuous modeling of the diffusion generative model though the stochastic differential equation, we refer the reader to it for the details. For a self-contained organization, we follow the configurations of the denoising diffusion probability model (DDPM). Following~\citep{ho2020denoising}, the process is formulated as a Markov chain: $p(\bm{x}_i|\bm{x}_{i-1})=\mathcal{N}(\bm{x}_i;\sqrt{1-\beta_i}\bm{x}_{i-1},\beta_i\bm{I})$, where $\{\beta_i\}_{i=1}^N$ is a sequence of noise levels with $0<\beta_i<1$. By recursion, one can show that $p(\bm{x}_i|\bm{x}_0)=\mathcal{N}(\bm{x}_i;\sqrt{\alpha_i}\bm{x}_0,(1-\alpha_i)\bm{I})$, where $\alpha_i:=\Pi_{j=1}^i(1-\beta_j)$.

To generate new samples, we aim to reverse this process. However, the exact reverse distribution $q(\bm{x}_{i-1}|\bm{x}_i)$ is intractable. Therefore, \citet{ho2020denoising} proposed to learn a Gaussian approximation $p_{\theta}(\bm{x}_{i-1}|\bm{x}_i)$ to approximate the true conditional probabilities $q(\bm{x}_{i-1}|\bm{x}_i)$. We known that
\begin{equation}
  q(\bm{x}_{i-1}|\bm{x}_i,\bm{x}_0) = p(\bm{x}_i|\bm{x}_{i-1},\bm{x}_0)\frac{p(\bm{x}_{i-1}|\bm{x}_0)}{p(\bm{x}_i|\bm{x}_0)} \text{ is a Gaussian}.
\end{equation}
Specially, the mean and variance are
\begin{equation*}
  \begin{aligned}
    \tilde{\mu}(\bm{x}_i,\bm{x}_0) &= \frac{\sqrt{1-\beta_i}(1-\alpha_{i-1})}{1-\alpha_i}\bm{x}_i + \frac{\sqrt{\alpha_{i-1}}}{1-\alpha_i}\bm{x}_0,\\
    \tilde{\sigma}^2 &= \frac{1-\alpha_{i-1}}{1-\alpha_i}\cdot\beta_i.
  \end{aligned}
\end{equation*}
Since $\bm{x}_0$ is unknown during sampling, we estimate $\bm{x}_0$ using $\bm{x}_t$. The estimate of $p(\bm{x}_0|\bm{x}_t)$ with given $p(\bm{x}_t|\bm{x}_0)$ is generally intractable, however, the mean of $p(\bm{x}_0|\bm{x}_i)$ can be computed using the Tweedie's formula
\begin{equation}\label{eq:twee0}
  \mathbb{E}[\bm{x}_0|\bm{x}_i] = \frac{\bm{x}_i + (1-\alpha_i)\nabla_{\bm{x}_i}\log p(\bm{x}_i)}{\sqrt{\alpha_i}}.
\end{equation}
Diffusion model aims to train a network $\bm{s}_{\theta}(\bm{x}_i,i)$ to predict the score function $\nabla_{\bm{x}_i}\log p(\bm{x}_i)$. Alternatively, DDPM trained a network $\bm{\epsilon}_{\theta}(\bm{x}_i,i)\simeq -\sqrt{1-\alpha_i}\nabla_{\bm{x}_i}\log p(\bm{x}_i)$. Hence
\begin{equation}\label{eq:twee}
  \mathbb{E}[\bm{x}_0|\bm{x}_i] \simeq  \frac{\bm{x}_i - \sqrt{1-\alpha_i}\bm{\epsilon}_{\theta}(\bm{x}_i,i)}{\sqrt{\alpha_i}}.
\end{equation}

Note that $\bm{x}_0 =\frac{\bm{x}_i - \sqrt{1-\alpha_i}\bm{z}}{\sqrt{\alpha_i}}$, comparing the formula to~\eqref{eq:twee}, hence the training loss
\begin{equation}\label{eq:ls0}
  \mathcal{L}:= \sum_{i}\frac{1-\alpha_i}{\alpha_i}\norm{\bm{\epsilon}_{\theta}(\bm{x}_i,i) - \bm{z}}_2^2,
\end{equation}
where $\bm{x}_i = \sqrt{\alpha_i}\bm{x}_0 + \sqrt{1-\alpha_i}\bm{z}$. When deep diffusion model is available, we can perform DDPM, DDIM~\citep{song2021denoising}, and other method to perform the unconditional sampling.

The training loss~\eqref{eq:ls0} actually can be expressed as
\begin{equation*}
  \mathcal{L}:= \sum_{i} \norm{\text{Denoise}[\bm{x}_i]-\bm{x}_0}_2^2,
\end{equation*}
where $\text{Denoise}[\bm{x}_i]$ is the MMSE denoiser, which is exactly the equation~\eqref{eq:twee}. Using this in a regularized optimization framework, we propose to solve:
\begin{equation}\label{eq:final_opt}
  \min_{\bm{x}} \quad \frac{1}{\lambda}\norm{\bm{A}\bm{x}-\bm{y}}_q^q + \sum_{i} \norm{\text{Denoise}[\bm{x}_i]-\bm{x}}_2^2,
\end{equation}
where $\bm{x}_i=\sqrt{\alpha_i}\bm{x}+\sqrt{1-\alpha_i}\bm{z}$ and $\lambda=\delta^q$. The parameters $\delta$ and $q$ are estimated from the gGSM model for the considered noise process of image restoration.

\subsection{Majorization-minimization Algorithm}

We aim to solve the optimization problem with the following formulation:
\begin{equation}\label{eq:mm}
  \min \mathcal{Q}(\bm{x}):=\frac{1}{\lambda}\norm{\bm{A}\bm{x}-\bm{y}}_q^q + \mathcal{R}(\bm{x}),
\end{equation}
where the definition $\norm{\bm{v}}_q^q = \sum_{i=1}^n |v_i|^q$ ($0<q\leq 2$) for a vector $\bm{v}\in\mathbb{R}^n$. Direct minimization is difficult due to the non-smooth and non-convex nature of the $\ell_q$-norm. To address this, we adopt a majorization-minimization (MM) approach~\citep{sun2016majorization}, which iteratively minimizes an upper-bound surrogate function. At each iteration $k$, we construct a surrogate $\mathcal{Q}^k(\bm{x}; \bm{x}_k)$ such that:
\begin{equation}
  \mathcal{Q}(\bm{x})\leq \mathcal{Q}^k(\bm{x};\bm{x}_k),\quad \mathcal{Q}(\bm{x}_k) = \mathcal{Q}^k(\bm{x}_k;\bm{x}_k).
\end{equation}
Minimizing $\mathcal{Q}^k$ guarantees:
\begin{equation}
  \mathcal{Q}(\bm{x}_{k+1})\leq \mathcal{Q}^k(\bm{x}_{k+1};\bm{x}_k) \leq \mathcal{Q}^k(\bm{x}_{k};\bm{x}_k) = \mathcal{Q}(\bm{x}_{k}),
\end{equation}
ensuring that $\mathcal{Q}(\bm{x}_k)$ is non-increasing. Under mild conditions, the sequence ${\bm{x}_k}$ converges to a stationary point of~\eqref{eq:mm}.

To majorize the $\norm{\bm{A}\bm{x}-\bm{y}}_q^q$ loss term, we use the following inequality: 
\begin{lemma}
  For any $x, y \in \mathbb{R}$ with $y \neq 0$ and $0 < q \leq 2$, the following inequality holds:
  \begin{equation}\label{eq:ine}
    |x|^q\leq \frac{q}{2}|y|^{q-2}x^2 + \frac{2-q}{2}|y|^q.
  \end{equation}
\end{lemma}

\begin{proof}
  Consider the function $f(x)=|x|^{q/2}$, by mean value theorem, the result holds by taking $x\leftarrow x^2$.
\end{proof}

\begin{lemma}\label{lem:2}
  For any $\bm{x}\in\mathbb{R}^d$, consider the function $f(\bm{x})=\norm{\bm{A}\bm{x}-\bm{y}}_q^q$, then we have
  \begin{equation}
    \norm{\bm{A}\bm{x}-\bm{y}}_q^q\leq \frac{q}{2}\norm{\bm{W}_k\circ (\bm{A}\bm{x}-\bm{y})}_2^2 + \text{const},
  \end{equation}
  where the constant term is independent of $\bm{x}$. $\bm{W}_k:=  ((\bm{A}\bm{x}_k-\bm{y})^2 +\bm{\epsilon}_k)^{\frac{q-2}{4}}$ is the reweighted matrix and $\circ$ denotes the element-wise (Hadamard) product.
\end{lemma}

\subsection{Iteratively Reweighted Least-squares}

Using the MM approach, we solve the regularized least-squares minimization at each iteration:
\begin{equation}\label{eq:ls}
  \min_{\bm{x}}\quad \frac{q}{2\lambda}\norm{\bm{W}_k\circ (\bm{A}\bm{x}-\bm{y})}_2^2 + \sum_{i} \norm{\text{Denoise}[\bm{x}_i]-\bm{x}}_2^2.
\end{equation}
To solve~\eqref{eq:ls}, we adopt the forward-backward splitting method~\citep{parikh2014proximal}, which alternates between a gradient descent step and a proximal step:
\begin{equation}\label{eq:sub2}
  \begin{aligned}
    \bm{x}_{t} &= \bm{x}_{t}- \frac{s_t\eta_k q}{\lambda}\bm{A}^T(\bm{W}_k^2\circ (\bm{A}\bm{x}_{t}-\bm{y}))\\
    \bm{x}_{t+1} &= \min_{\bm{x}}\quad \frac{1}{2\eta_k}\norm{\bm{x}-\bm{x}_{t}}_2^2 + \sum_{i} \norm{\text{Denoise}[\bm{x}_i]-\bm{x}}_2^2,
  \end{aligned}
\end{equation}
where $s_t$ is a tunable stepsize. For convergence, it is typically required that $s_t \eta_k < \frac{2}{L}$, where $L$ is the Lipschitz constant of the gradient of the least-squares term.

Solving the proximal step exactly requires computing the gradient of the regularization term, which involves backpropagation through the denoiser—this is computationally expensive. While some specific formulations allow for gradient simplification or avoidance~\citep{mardanivariational}, evaluating the full summation remains costly. To mitigate this, we adopt a plug-and-play (PnP) strategy: approximate the proximal step using a denoiser derived from the diffusion model.

To align the noise level $\eta_k$ with the noise distribution expected by the diffusion model, we set $\eta_k = (1 - \alpha_k)/\alpha_k$. However, the intermediate iterate $\bm{x}_t$ generally does not match the input distribution for which the diffusion model $\bm{\epsilon}\theta$ is trained. A common remedy is to renoise $\bm{x}_t$ by adding Gaussian noise:
\begin{equation*}
  \begin{aligned}
    \bm{x}_t' &= \sqrt{\alpha_k}\bm{x}_t + \sqrt{1-\alpha_k}\bm{z}\\
    \bm{x}_{t+1} & = \text{Denoise}[\bm{x}_t'].
  \end{aligned}
\end{equation*}
The `renoising step' is often referred to as time reversal, has been widely used in diffusion-based conditional generation for image restoration~\citep{choi2021ilvr, wang2022zero}.

\setlength{\textfloatsep}{8pt}
\begin{algorithm}[!ht]
\caption{Image restoration with general noise removal via plug-and-play diffusion prior.\label{alg:2}}
\begin{algorithmic}[1]
   \REQUIRE {Outer iterations $T$, pretrained diffusion model $\bm{\epsilon}_{\theta}$ and the noise schedule $\{\alpha_t\}$, hyperparameter $q$ ($0<q\leq 2$), pre-defined noise level $\{\eta_t=\frac{1-\alpha_t}{\alpha_t}\}$, perturbation sequences $\{\epsilon_t\}$, inter iteration $T_{inter}$}
 \ENSURE {Estimated  image $\bm{x}_0$.}
 \STATE Set $\bm{x}_T=\bm{z}\sim \mathcal{N}(0,\bm{I})$
 \FOR{$t = T:-1:1$}
  \STATE Set $\bm{x}_0 = \bm{x}_t$ and the reweighted matrix\\
    $\bm{W}_t = ((\bm{A}\bm{x}_t-\bm{y})^2+\epsilon_t)^{q-2}$
  \STATE \%\% Solve MM~\eqref{eq:ls} via forward-backward splitting
  \FOR{$i= 0:1:T_{inter}-1$}
  \STATE \%\% Gradient descent step:\\
  $\bm{x}_i=\bm{x}_i-s_i\eta_t\bm{A}^T(\bm{W}_t(\bm{A}\bm{x}_i-\bm{y}))$
  \STATE \%\% Denoising step:\\
  \begin{equation*}
    \begin{aligned}
      \bm{x}_i' &= \sqrt{\alpha_t}\bm{x}_i + \sqrt{1-\alpha_t}\bm{z}\\
      \bm{x}_{i+1}& = \frac{\bm{x}_i'-\sqrt{1-\alpha_t}\bm{\epsilon}_{\theta}(\bm{x}_i',t)}{\sqrt{\alpha_t}}
    \end{aligned}
  \end{equation*}
  \ENDFOR
  \STATE \%\% Set the update
  $\bm{x}_{t-1}=\bm{x}_{T_{inter}}$
\ENDFOR
\end{algorithmic}
\end{algorithm}

\noindent{\bf Overflow avoidance via decreasing perturbation $\bm{\epsilon}_k$}\ \

To prevent overflow, we introduce a small perturbation $\epsilon_k > 0$ in the weighted matrix $\bm{W}_k= ((\bm{A}\bm{x}_k-\bm{y})^2+\epsilon_k)^{\frac{q-2}{2}}$. Moreover, the choice of $\epsilon_k$ affects both convergence speed and the sparsity of the solution (due to the nonconvex nature of the $\ell_q$ term). Following~\citep{chartrand2008iteratively}, we employ a decreasing perturbation schedule ${\epsilon_k}$, which both stabilizes early iterations and improves sparsity in the later stage.

To summarize, we propose an image restoration method that addresses general noise via an $\ell_q$-norm fidelity term, combined with a plug-and-play diffusion prior. The resulting iteratively reweighted least-squares problem is solved using forward-backward splitting. The proximal step leverages the denoiser induced by a diffusion generative model, with careful renoising to ensure effective denoising. The complete algorithm is presented in Algorithm~\ref{alg:2}.

\begin{figure}[!htbp]
   \centering \vspace*{-5pt}
   \begin{tabular}{c@{\hspace*{3pt}}c@{\hspace*{1pt}}c@{\hspace*{1pt}}c@{\hspace*{1pt}}c@{\hspace*{1pt}}c@{\hspace*{1pt}}c@{\hspace*{1pt}}c@{\hspace*{1pt}}c@{\hspace*{1pt}}c@{\hspace*{2pt}}c@{\hspace*{2pt}}c@{\hspace*{2pt}}c@{\hspace*{1pt}}}
 \rotatebox[origin=c]{90}{Deblur} &\raisebox{-0.5\height}{\includegraphics[width = 0.11\textwidth]{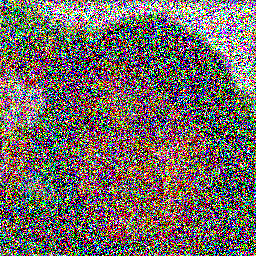}}
 & \raisebox{-0.5\height}{\includegraphics[width = 0.11\textwidth]{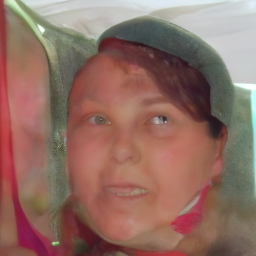}}
 & \raisebox{-0.5\height}{\includegraphics[width = 0.11\textwidth]{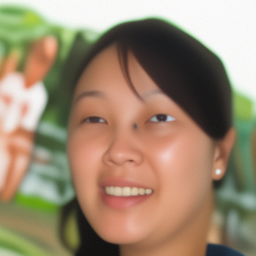}}
     & \raisebox{-0.5\height}{\includegraphics[width = 0.11\textwidth]{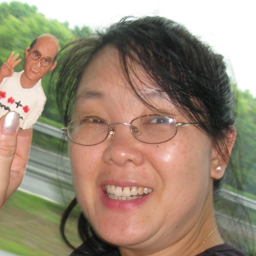}}\\[-2pt]
     \rotatebox[origin=c]{90}{Inpainting} &\raisebox{-0.5\height}{\includegraphics[width = 0.11\textwidth]{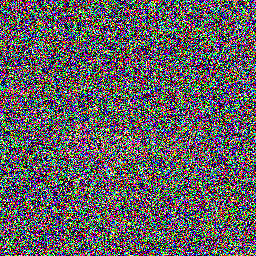}}
 & \raisebox{-0.5\height}{\includegraphics[width = 0.11\textwidth]{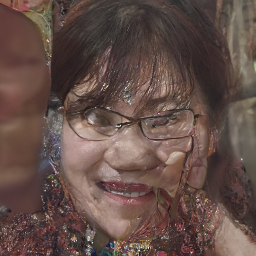}}
 & \raisebox{-0.5\height}{\includegraphics[width = 0.11\textwidth]{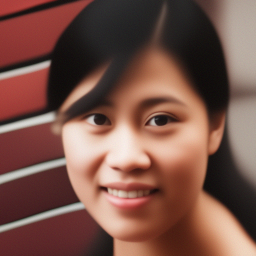}}
 & \raisebox{-0.5\height}{\includegraphics[width = 0.11\textwidth]{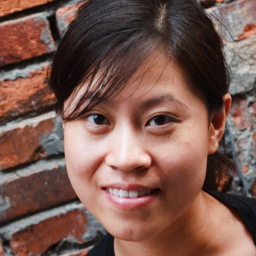}}\\
    &  \small{Input} & \small{DPS variant} & \small{Ours} & \small{GT}\\
 \end{tabular}\vspace*{-8pt}
 \caption{\textbf{Visualization of different methods on two tasks with impulse noise for FFHQ. The DPS equipped with the IRLS loss produces artifact-corrupted restoration.}\label{fig:dps_comp} } \vspace*{-15pt}
 \end{figure} 

\subsection{Updated Losses for DPS Method}

The $\ell_q$-norm data fidelity term can also be incorporated into the existing Diffusion Posterior Sampling (DPS) framework~\citep{chung2023fast} to address image restoration with general noise corruption.

To adapt DPS for impulse noise, we consider two straightforward modifications to the standard DPS pipeline, which is originally designed for Gaussian noise. These modifications alter the measurement guidance step in DPS, where the transition from unconditional to conditional sampling occurs. The first modification directly applies gradient descent on the $\ell_q$-norm fidelity term during the measurement guidance step. Specifically, we perform the following updates:
\begin{equation*}
  \begin{aligned}
    \bm{x}_{t-1}&=\text{UncondSample}(\bm{x}_t,\bm{\epsilon}_{\theta}(\bm{x}_t,t))\\
    \bm{x}_{t-1}&=\bm{x}_{t-1}-\rho_t\nabla_{\bm{x}_t}\norm{\bm{A}\hat{\bm{x}}_0(\bm{x}_t)-\bm{y}}_q^q,
  \end{aligned}
\end{equation*}
where $\hat{\bm{x}}_0(\bm{x}_t)$ is the estimate of $\mathbb{E}[\bm{x}_0|\bm{x}_t]$ using Tweedie's formula. However, this modification suffers from poor performance due to the nonconvexity of the $\ell_q$-norm (especially for $q < 1$), which leads to unstable and suboptimal reconstructions. 

Alternatively, following the IRLS of our approach, one can perform the measurement guidance step with the following step:
\begin{equation}
  \bm{x}_{t-1}=\bm{x}_{t-1}-\rho_t\nabla_{\bm{x}_t}\norm{\bm{W}_t\circ (\bm{A}\hat{\bm{x}}_0(\bm{x}_t)-\bm{y})}_2^2,
\end{equation}
where $\bm{W}_t = ((\bm{A}\bm{x}_t-\bm{y})^2+\epsilon_t)^{\frac{q-2}{4}}$. This modified DPS method performs better than the naive $\ell_q$ gradient approach and yields feasible restoration results. However, it still suffers from visible artifacts and reduced fidelity compared to our full IRLS-based plug-and-play method. Comparative results are provided in Figure~\ref{fig:dps_comp}.

\subsection{Applicability for Other Generative Models}

Our method integrated a plug-and-play (PnP) diffusion prior with an $\ell_q$-norm fidelity term to address general noise. As a generalizable PnP framework, its core innovation lies in how the denoiser (induced by a generative model) is incorporated into the iterative solver. Most PnP methods differ primarily in their choice of denoiser. Conventional denoisers are typically trained for Gaussian noise at a fixed noise level, requiring paired datasets. In contrast, diffusion models naturally yield a continuum of denoisers across a range of noise levels. To exploit this, we introduce a renoising step that ensures the input to the denoiser matches the Gaussian noise distribution expected by the diffusion model.

Due to the modularity of PnP models, all available generative models can be integrated in our proposed framework. Hence, given any generative model, including the flow matching model~\citep{lipmanflow}, diffusion model~\citep{ho2020denoising}, and rectified flow model~\citep{liuflow}, we can follow the algorithmic flowchart to restore images from its degradations. The minor modification is the replacement of the denoiser induced by the adopted  generative models. To show the flexibility of our algorithm, we consider image restoration with general noise removal with integration of the diffusion model, and the flow matching model.

\section{Numerical Experiments}

\begin{figure*}[!ht]
    \centering 
    \begin{tabular}{c@{\hspace*{3pt}}c@{\hspace*{1pt}}c@{\hspace*{1pt}}c@{\hspace*{1pt}}c@{\hspace*{1pt}}c@{\hspace*{1pt}}c@{\hspace*{1pt}}c@{\hspace*{1pt}}c@{\hspace*{1pt}}c@{\hspace*{2pt}}c@{\hspace*{2pt}}c@{\hspace*{2pt}}c@{\hspace*{1pt}}}
  \rotatebox[origin=c]{90}{Denoising} &\raisebox{-0.5\height}{\includegraphics[width = 0.15\textwidth]{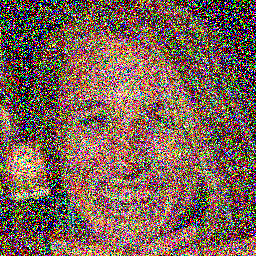}}
  & \raisebox{-0.5\height}{\includegraphics[width = 0.15\textwidth]{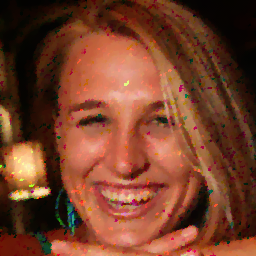}}
  & \raisebox{-0.5\height}{\includegraphics[width = 0.15\textwidth]{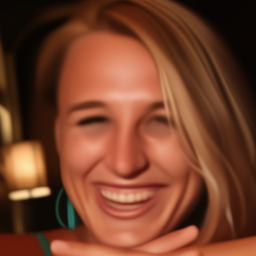}}
  & \raisebox{-0.5\height}{\includegraphics[width = 0.15\textwidth]{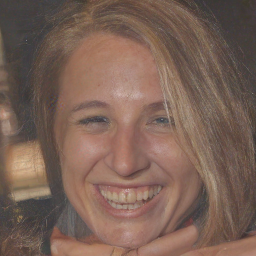}}
  & \raisebox{-0.5\height}{\includegraphics[width = 0.15\textwidth]{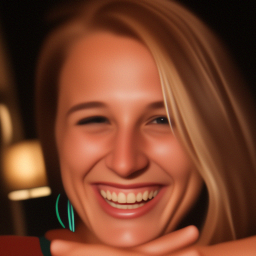}}
  & \raisebox{-0.5\height}{\includegraphics[width = 0.15\textwidth]{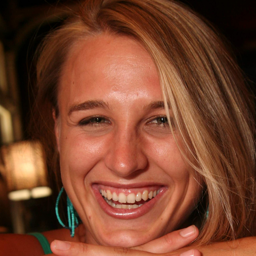}}\\[-2pt]
\rotatebox[origin=c]{90}{Deblur} &\raisebox{-0.5\height}{\includegraphics[width = 0.15\textwidth]{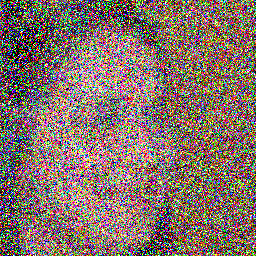}}
& \raisebox{-0.5\height}{\includegraphics[width = 0.15\textwidth]{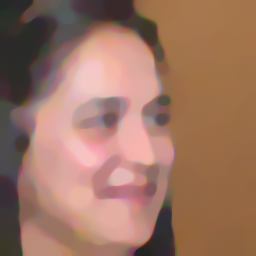}}
& \raisebox{-0.5\height}{\includegraphics[width = 0.15\textwidth]{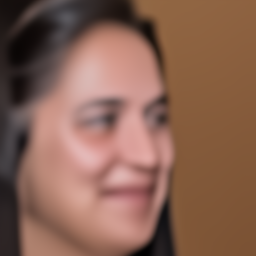}}
& \raisebox{-0.5\height}{\includegraphics[width = 0.15\textwidth]{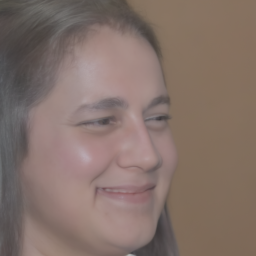}}
& \raisebox{-0.5\height}{\includegraphics[width = 0.15\textwidth]{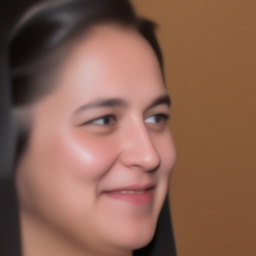}}
& \raisebox{-0.5\height}{\includegraphics[width = 0.15\textwidth]{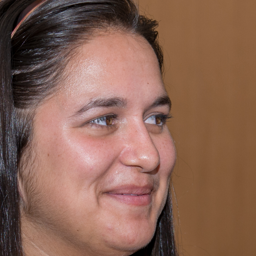}}\\[-2pt]
\rotatebox[origin=c]{90}{Inpaint} &\raisebox{-0.5\height}{\includegraphics[width = 0.15\textwidth]{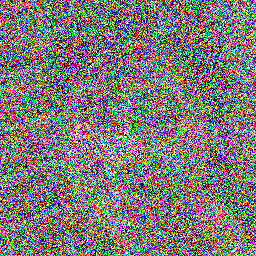}}
& \raisebox{-0.5\height}{\includegraphics[width = 0.15\textwidth]{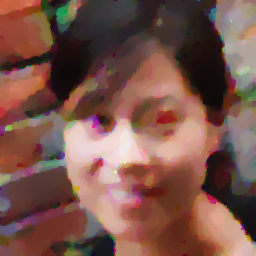}}
& \raisebox{-0.5\height}{\includegraphics[width = 0.15\textwidth]{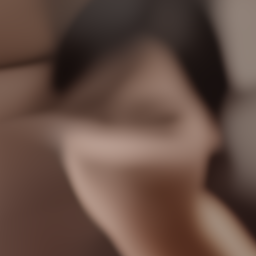}}
& \raisebox{-0.5\height}{\includegraphics[width = 0.15\textwidth]{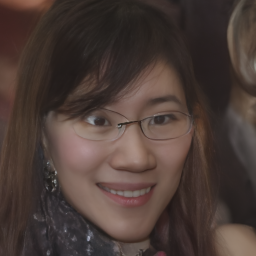}}
& \raisebox{-0.5\height}{\includegraphics[width = 0.15\textwidth]{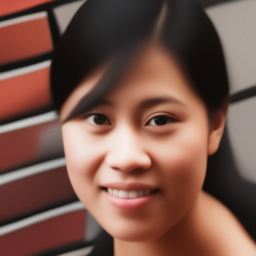}}
& \raisebox{-0.5\height}{\includegraphics[width = 0.15\textwidth]{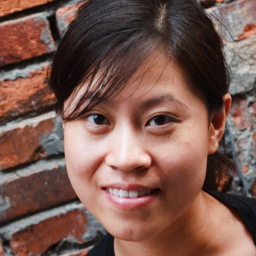}}\\[-2pt]
\rotatebox[origin=c]{90}{SR} &\raisebox{-0.5\height}{\includegraphics[width = 0.15\textwidth]{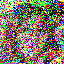}}
& \raisebox{-0.5\height}{\includegraphics[width = 0.15\textwidth]{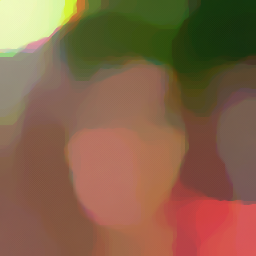}}
& \raisebox{-0.5\height}{\includegraphics[width = 0.15\textwidth]{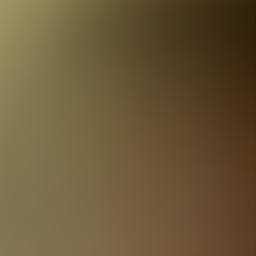}}
& \raisebox{-0.5\height}{\includegraphics[width = 0.15\textwidth]{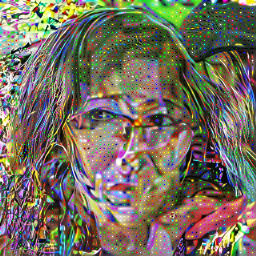}}
& \raisebox{-0.5\height}{\includegraphics[width = 0.15\textwidth]{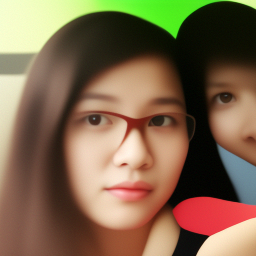}}
& \raisebox{-0.5\height}{\includegraphics[width = 0.15\textwidth]{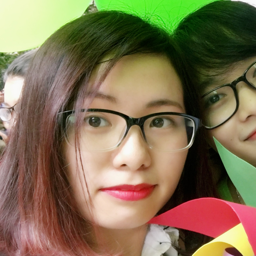}}\\[-2pt]
  \rotatebox[origin=c]{90}{Denoising} &\raisebox{-0.5\height}{\includegraphics[width = 0.15\textwidth]{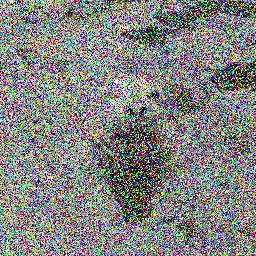}}
  & \raisebox{-0.5\height}{\includegraphics[width = 0.15\textwidth]{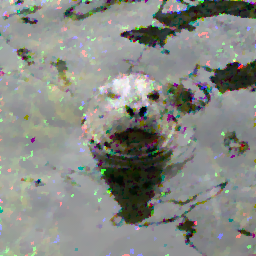}}
  & \raisebox{-0.5\height}{\includegraphics[width = 0.15\textwidth]{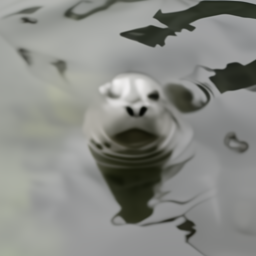}}
  & \raisebox{-0.5\height}{\includegraphics[width = 0.15\textwidth]{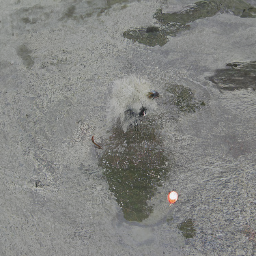}}
  & \raisebox{-0.5\height}{\includegraphics[width = 0.15\textwidth]{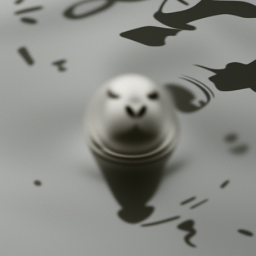}}
      & \raisebox{-0.5\height}{\includegraphics[width = 0.15\textwidth]{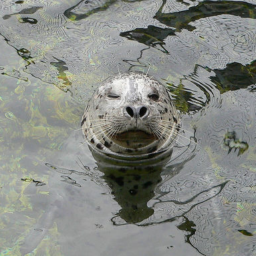}}\\[-2pt]
   \rotatebox[origin=c]{90}{Deblur} &\raisebox{-0.5\height}{\includegraphics[width = 0.15\textwidth]{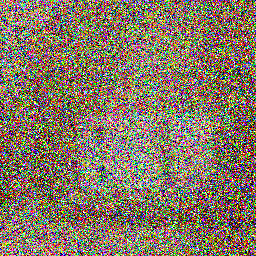}}
  & \raisebox{-0.5\height}{\includegraphics[width = 0.15\textwidth]{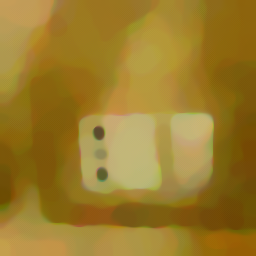}}
  & \raisebox{-0.5\height}{\includegraphics[width = 0.15\textwidth]{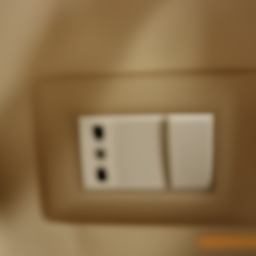}}
  & \raisebox{-0.5\height}{\includegraphics[width = 0.15\textwidth]{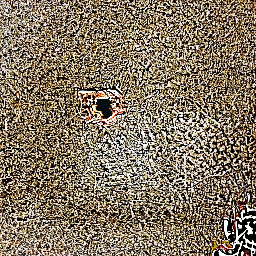}}
  & \raisebox{-0.5\height}{\includegraphics[width = 0.15\textwidth]{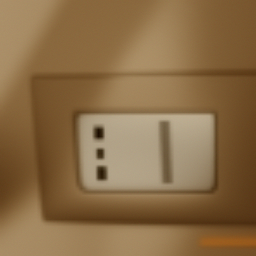}}
      & \raisebox{-0.5\height}{\includegraphics[width = 0.15\textwidth]{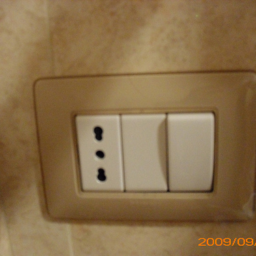}}\\[-2pt]
   \rotatebox[origin=c]{90}{Inpaint} &\raisebox{-0.5\height}{\includegraphics[width = 0.15\textwidth]{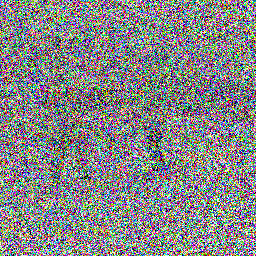}}
  & \raisebox{-0.5\height}{\includegraphics[width = 0.15\textwidth]{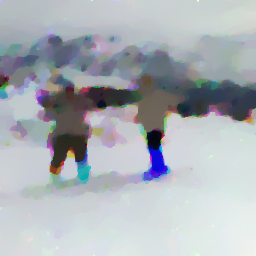}}
  & \raisebox{-0.5\height}{\includegraphics[width = 0.15\textwidth]{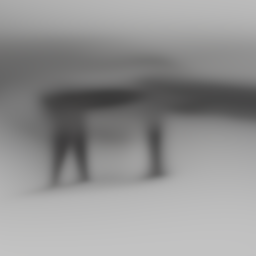}}
  & \raisebox{-0.5\height}{\includegraphics[width = 0.15\textwidth]{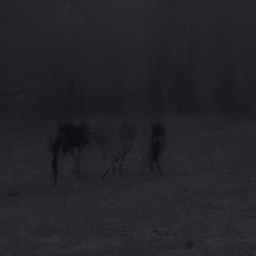}}
  & \raisebox{-0.5\height}{\includegraphics[width = 0.15\textwidth]{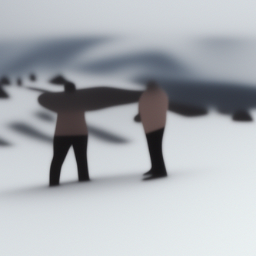}}
  & \raisebox{-0.5\height}{\includegraphics[width = 0.15\textwidth]{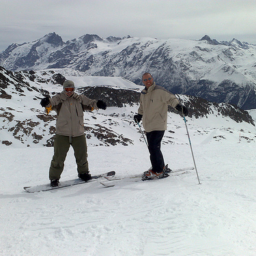}}\\ 
  &  \small{Input} & \small{PnP-TV} & \small{DRUnet} & \small{DPS} & \small{Ours} & \small{GT}\\
  \end{tabular}\vspace*{-2pt}
  \caption{\textbf{Visualization of different methods on four tasks with impulse noise for FFHQ and ImageNet. }\label{fig:ffhq_} } \vspace*{-6pt}
  \end{figure*} 

To evaluate the performance of our algorithm for image restoration with general noise via gGSM approximation, we conduct experiments on various image degradations with simulated impulse salt-and-pepper noise. We consider three benchmark datasets with image size $256\times 256$, including Flickr Faces High Quality (FFHQ)~\citep{karras2019style}, ImageNet~\citep{deng2009imagenet}, and Animal Faces HQ (AFHQ) with cat category. The publicly available generative models for the first two datasets are trained using diffusion model, and both of them are built on VP SDE. For FFHQ, we use the pretrained DDPM model from~\citep{choi2021ilvr}, and for ImageNet, we use the pretrained DDPM model from~\citep{dhariwal2021diffusion}. For the AFHQ-Cat dataset, we access the pretrained flow matching model~\citep{martin2024pnp} with a time point definition conversion to follow the convention of flow model. All models are used as-is without task-specific fine-tuning. For FFHQ and ImageNet, we randomly select 100 testing images form 1K tested images to evaluate the algorithm, while for AFHQ, we report the result on 32 testing images. The induced denoiser of diffusion and flow model can be obtained using Tweedie's formula. 

\noindent{\bf Evaluated imaging problems}\ \
The considered four image restoration problems are as follows. (a) Denoising; (b) deblurring using a $61\times 61$ Gaussian kernel; (c) random pixel inpainting with $70\%$ missing  pixels. (d) supper-resolution with $4\times$ downsampling via average pooling.

\noindent{\bf Simulated noisy measurement}\ \
To evaluate the advantages of $\ell_q$-norm loss for simulated impulse noise. We consider corrupting the measurement with salt-and-pepper noise. We consider two noise levels $r_{sp}=0.5,0.7$, which means that the random $50\%,70\%$ percentage of pixels are corrupted by setting the values to the minimum and maximum values with equality probability.

\noindent{\bf Quantitative results}\ \
Within the PnP solving framework, fixing the data fidelity term, we evaluate the effect of different regularizations on the image restorations. The compared regularizations include total variation (TV), DRUnet~\citep{zhang2021plug} pretrained for Gaussian noise removal, and our method with denoiser induced by diffusion  and flow model. We also include the original DPS method~\citep{chung2022diffusion} with $q=2$ as a baseline. All methods share a consistent IRLS-based solver and optimization setup. The stepsize is set as $s_i = 1/\norm{\bm{A}^T(\bm{W}_t(\bm{A}\bm{x}_i-\bm{y}))}_2$ and $T_{inter}=1$. For TV regularization, we run 60-step proximal minimization and the trade-off parameter is set to $s_i\eta_t$. For DRUnet, the noise level is set to $s_i\eta_t$ as well. For all experiments, we set $q=0.5$. We evaluate the results using PSNR$\uparrow$ and SSIM$\uparrow$ to quantify fidelity to the original images.

Table~\ref{tab:salt} presents quantitative results under $r_{\text{sp}} = 0.5$ salt-and-pepper noise. Across all datasets and tasks, our method consistently outperforms all competing approaches. Visual comparisons in Figures~\ref{fig:ffhq_} and~\ref{fig:afhq_} highlight the superior reconstruction quality, especially for challenging tasks like inpainting and super-resolution. DPS achieves visually detailed outputs but suffers from measurement inconsistency, often producing a gray layer over the image. These results underscore the effectiveness of the $\ell_q$-norm data fidelity in handling impulse noise and the superiority of generative priors (diffusion and flow) over hand-crafted or conventional denoisers. Additional qualitative examples are provided in the Appendix.

\begin{table}[!ht]
  \centering
  \small
    \fontsize{9}{12}\selectfont 
  \caption{Quantitative results of different methods for different tasks on the  three datasets with $0.5$ salt-and-pepper noise. \label{tab:salt}}
 \begin{tabular}{c@{\hspace{3pt}}c@{\hspace{1pt}}c@{\hspace{1pt}}c@{\hspace{1pt}}c@{\hspace{1pt}}c@{\hspace{1pt}}c@{\hspace{1pt}}c@{\hspace{1pt}}c@{\hspace{1pt}}c@{\hspace{1pt}}c@{\hspace{1pt}}c@{\hspace{1pt}}c@{\hspace{1pt}}c@{\hspace{1pt}}c@{\hspace{1pt}}c@{\hspace{1pt}}cc@{\hspace{1pt}}c@{\hspace{1pt}}c@{\hspace{0pt}}cc@{\hspace{1pt}}c@{\hspace{1pt}}c@{\hspace{0pt}}cc@{\hspace{1pt}}c@{\hspace{1pt}}c@{\hspace{0pt}}}
 \toprule
  & \multirow{2}{*}{Method} &   \multicolumn{2}{c}{\textbf{Denoise}}  && \multicolumn{2}{c}{\textbf{Deblur}} && \multicolumn{2}{c}{\textbf{Inpaint}}  && \multicolumn{2}{c}{\textbf{SR}} \\
        \cline{3-4} \cline{6-7} \cline{9-10} \cline{12-13}
   &  & PSNR  & SSIM && PSNR  & SSIM  && PSNR  & SSIM && PSNR  & SSIM \\
   \toprule
 \multirow{4}{*}{\STAB{\rotatebox[origin=c]{90}{FFHQ}}} & PnP-TV     & 25.22 & 0.763 &  & 20.39 & 0.633 &  & 22.74 & 0.686  &  & 16.82 & 0.552 \\
& DRUnet & 26.42 & 0.765  &  & 25.10  & 0.719 &  & 16.33 & 0.572 &  & 12.90  & 0.464 \\
& DPS        & 16.71 & 0.565 &  & 16.36 & 0.549 &  & 11.25 & 0.416 &  & 10.83 & 0.065 \\
& Ours    & 28.36 & 0.828  &  & 25.45 & 0.729 &  & 23.59 & 0.727 &  & 20.24 & 0.621\\
\midrule
\multirow{4}{*}{\STAB{\rotatebox[origin=c]{90}{ImageNet}}} & PnP-TV     & 22.84 & 0.683 &  & 18.20  & 0.491 &  & 21.70  & 0.591 &  & 14.13 & 0.436 \\
& DRUnet & 24.63 & 0.675  &  & 22.68 & 0.589 &  & 17.08 & 0.506 &  & 10.45 & 0.237  \\
& DPS        & 16.36 & 0.412 &  & 12.32 & 0.215 &  & 7.96  & 0.248 &  & 8.97  & 0.037 \\
& Ours       & 25.81 & 0.679 &  & 23.33 & 0.613 &  & 21.71 & 0.593 &  & 16.36 & 0.431 \\
   \midrule

 \multirow{3}{*}{\STAB{\rotatebox[origin=c]{90}{AFHQ}}} &  PnP-TV     & 24.59 & 0.729 &  & 20.16 & 0.560  &  & 23.21 & 0.643 &  & 16.98 & 0.475 \\
& DRUnet & 27.43 & 0.770  &  & 24.74 & 0.654 &  & 16.12 & 0.505 &  & 14.00    & 0.434 \\
& Ours       & 29.73 & 0.879 &  & 25.26 & 0.677 &  & 26.16 & 0.761 &  & 20.94 & 0.581\\
   \bottomrule
\end{tabular}
\end{table}

\begin{figure}[!ht]
   \centering 
   \begin{tabular}{c@{\hspace*{3pt}}c@{\hspace*{1pt}}c@{\hspace*{1pt}}c@{\hspace*{1pt}}c@{\hspace*{1pt}}c@{\hspace*{1pt}}c@{\hspace*{1pt}}c@{\hspace*{1pt}}c@{\hspace*{1pt}}c@{\hspace*{2pt}}c@{\hspace*{2pt}}c@{\hspace*{2pt}}c@{\hspace*{1pt}}}
\rotatebox[origin=c]{90}{Input} &\raisebox{-0.5\height}{\includegraphics[width = 0.105\textwidth]{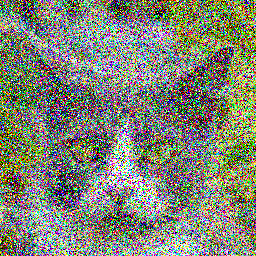}}      & \raisebox{-0.5\height}{\includegraphics[width = 0.105\textwidth]{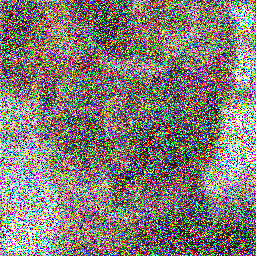}}      & \raisebox{-0.5\height}{\includegraphics[width = 0.105\textwidth]{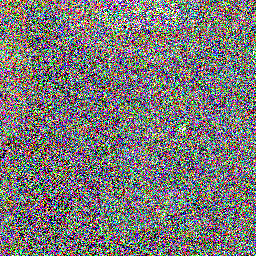}}      & \raisebox{-0.5\height}{\includegraphics[width = 0.105\textwidth]{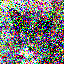}} \\[-2pt]
     \rotatebox[origin=c]{90}{PnP-TV} &\raisebox{-0.5\height}{\includegraphics[width = 0.105\textwidth]{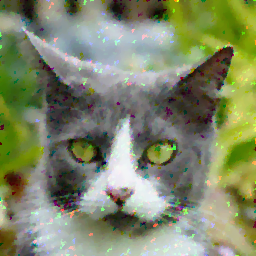}}         & \raisebox{-0.5\height}{\includegraphics[width = 0.105\textwidth]{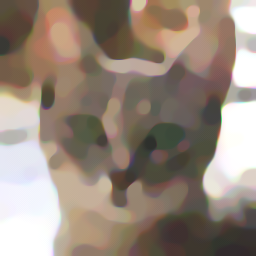}}         & \raisebox{-0.5\height}{\includegraphics[width = 0.105\textwidth]{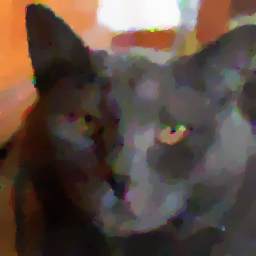}}         & \raisebox{-0.5\height}{\includegraphics[width = 0.105\textwidth]{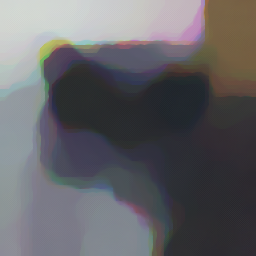}}    \\[-2pt]
     \rotatebox[origin=c]{90}{DRUnet} &\raisebox{-0.5\height}{\includegraphics[width = 0.105\textwidth]{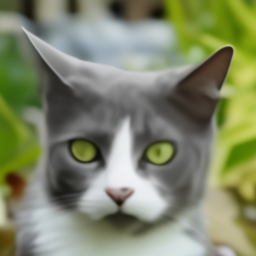}}     & \raisebox{-0.5\height}{\includegraphics[width = 0.105\textwidth]{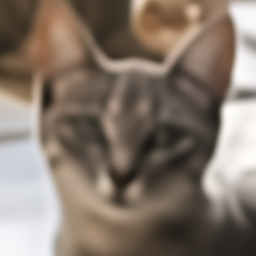}}     & \raisebox{-0.5\height}{\includegraphics[width = 0.105\textwidth]{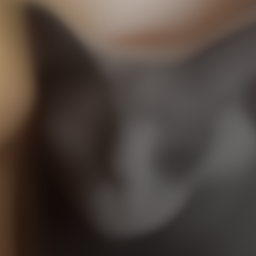}}     & \raisebox{-0.5\height}{\includegraphics[width = 0.105\textwidth]{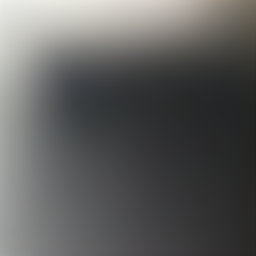}}\\[-2pt]
     \rotatebox[origin=c]{90}{Ours} &\raisebox{-0.5\height}{\includegraphics[width = 0.105\textwidth]{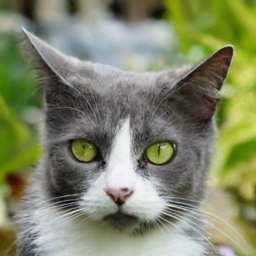}}       & \raisebox{-0.5\height}{\includegraphics[width = 0.105\textwidth]{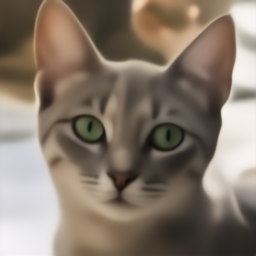}}       & \raisebox{-0.5\height}{\includegraphics[width = 0.105\textwidth]{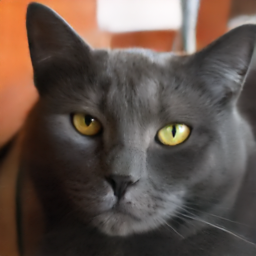}}       & \raisebox{-0.5\height}{\includegraphics[width = 0.105\textwidth]{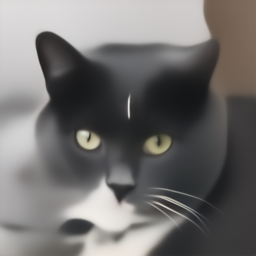}}  \\[-2pt]
     \rotatebox[origin=c]{90}{GT} &\raisebox{-0.5\height}{\includegraphics[width = 0.105\textwidth]{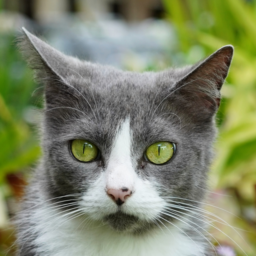}}         & \raisebox{-0.5\height}{\includegraphics[width = 0.105\textwidth]{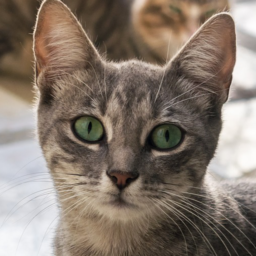}}         & \raisebox{-0.5\height}{\includegraphics[width = 0.105\textwidth]{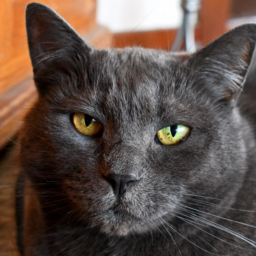}}         & \raisebox{-0.5\height}{\includegraphics[width = 0.105\textwidth]{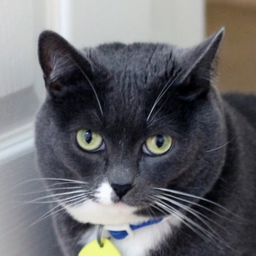}}\\
     &  \small{Denoising} & \small{Deblur} & \small{Inpainting} & \small{SR} \\
 \end{tabular}\vspace*{-2pt}
 \caption{\textbf{Visualization of different methods on four tasks with impulse noise for AFHQ-Cat.}\label{fig:afhq_} } 
\end{figure}

\noindent{\bf Ablation study on the choice of loss}\ \
We conduct an ablation study using AFHQ dataset to investigate the impact of the choice of $q$ in the $\ell_q$-norm data fidelity term on restoration performance. Our iteratively reweighed least-squares solver is applicable for any $q \in (0, 2]$. As shown in Introduction, for impulse noise removal, setting $q$ in interval $(0,1]$ is better. See Table~\ref{tab:time} for the final PSNR/SSIM results with different $q$ values. Smaller $q$ values (e.g., $q \in (0, 0.5]$) consistently yield better performance, and the performance degrades as $q$ increases beyond 0.5. See Figure~\ref{fig:abl} for the illustration of the effect of $q$ values.
\begin{table}[!ht]
  \centering 
  \small
    \fontsize{9}{12}\selectfont 
  \caption{Performance comparison of image denoising on AFHQ-Cat dataset using different $q$ values. \label{tab:time}} 
 \begin{tabular}{@{\hspace{2pt}}c@{\hspace{3pt}}c@{\hspace{3pt}}c@{\hspace{3pt}}c@{\hspace{3pt}}c@{\hspace{3pt}}c@{\hspace{3pt}}c@{\hspace{1pt}}c@{\hspace{1pt}}c@{\hspace{1pt}}c@{\hspace{1pt}}c@{\hspace{1pt}}c@{\hspace{1pt}}c@{\hspace{1pt}}c@{\hspace{1pt}}c@{\hspace{1pt}}c@{\hspace{1pt}}cc@{\hspace{1pt}}c@{\hspace{1pt}}c@{\hspace{0pt}}cc@{\hspace{1pt}}c@{\hspace{1pt}}c@{\hspace{0pt}}cc@{\hspace{1pt}}c@{\hspace{1pt}}c@{\hspace{0pt}}}
 \toprule

   $q$  & 0.2   & 0.3   & 0.4   & 0.5   & 0.6   & 0.7   & 0.8   & 0.9   & 1.0     \\
    \toprule
   PSNR & 29.62 & 29.69 & 29.72 & 29.73 & 29.66 & 29.49 & 29.16 & 28.59 & 27.76 \\
SSIM & 0.871 & 0.876 & 0.879 & 0.879 & 0.863 & 0.852 & 0.838 & 0.82  & 0.799\\
\bottomrule
\end{tabular}\vspace*{-10pt}
\end{table}

\begin{figure}[!ht]
   \centering 
   \begin{tabular}{c@{\hspace*{3pt}}c@{\hspace*{1pt}}c@{\hspace*{1pt}}c@{\hspace*{1pt}}c@{\hspace*{1pt}}c@{\hspace*{1pt}}c@{\hspace*{1pt}}c@{\hspace*{1pt}}c@{\hspace*{1pt}}c@{\hspace*{2pt}}c@{\hspace*{2pt}}c@{\hspace*{2pt}}c@{\hspace*{1pt}}}
      \includegraphics[width = 0.105\textwidth]{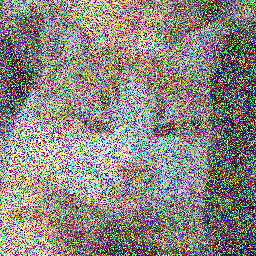} &
      \includegraphics[width = 0.105\textwidth]{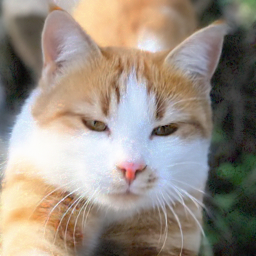} &
                                                                                                 \includegraphics[width = 0.105\textwidth]{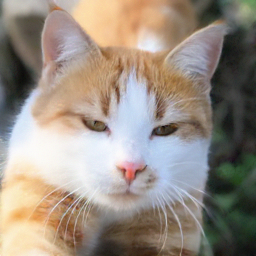}\\[-2pt]
     \small{Input } & \small{0.1} & \small{0.5}\\
      \includegraphics[width = 0.105\textwidth]{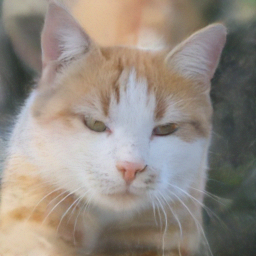} &
      \includegraphics[width = 0.105\textwidth]{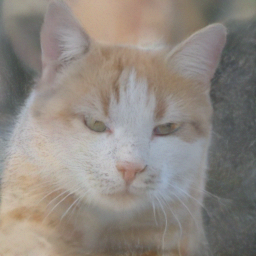} &
      \includegraphics[width = 0.105\textwidth]{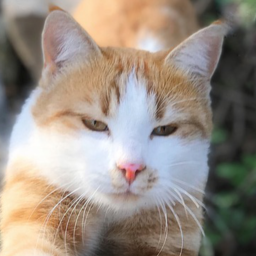} \\
     \small{1.5} & \small{2.0} & \small{GT}\\
 \end{tabular}\vspace*{-8pt}
 \caption{\textbf{Visualization of restorations from different $q$ values for image denoising on AFHQ-Cat.}\label{fig:abl} } 
\end{figure}

\section{Conclusion}

Within the general play-and-play (PnP) framework for image restoration, we propose using $\ell_q$ norm fidelity term to address general noise. Leveraging the strong diffusion prior as a regularization, we utilize the iteratively reweighted least-squares based method to efficiently solve  the diffusion prior regularized $\ell_q$ norm minimization problem. The core subproblem is tackled via a forward-backward splitting method, where the proximal step involving the diffusion prior is approximated through a denoising operation. This denoising is efficiently performed using the induced denoiser of the pretrained diffusion model. Experiments are conducted to evaluate our proposed method on three datasets with four imaging tasks, utilizing both diffusion and flow-based generative models as priors. Quantitative results and visualizations showcase the superior advantages of our method with the integration of plug-and-play diffusion prior and $\ell_q$ norm loss over other TV regularizer and deep denoiser. Our method also outperforms the original DPS with $q=2$, as it is tailored for Gaussian noise. 

\section{Acknowledgments}
JL was supported by the National Natural Science Foundation of China (Grant No. 12571472) and supported by the Open Project of Key Laboratory of Mathematics and Information Networks (Beijing University of Posts and Telecommunications), Ministry of Education, China, under Grant No. KF202401).


\newpage
\clearpage
\setcounter{page}{1}
\maketitlesupplementary

\appendix

\section{Proof for Lemma~\ref{lem:2} }
\label{sec:ao2}

\begin{proof}
  To avoid the possible overflow of $y=0$, we modify the inequality~\eqref{eq:ine} with the following form
  \begin{equation}
    |x|^q\leq \frac{q}{2}(y^2+\epsilon)^{\frac{q-2}{2}}x^2 + \frac{2-q}{2}(y^2+\epsilon)^{\frac{q}{2}}.
  \end{equation}
  Therefore, we have 
  \begin{equation*}
    \begin{aligned}
      &\norm{\bm{A}\bm{x}-\bm{y}}_q^q = \sum_{i=1}^m|(\bm{A}\bm{x}-\bm{y})_i|^q \\
      & \leq \sum_{i=1}^m\left(\frac{q}{2}((\bm{A}\bm{x}_k-\bm{y})_i^2+\epsilon_k)^{\frac{q-2}{2}}|(\bm{A}\bm{x}-\bm{y})_i|^2 \right .\\
        & \phantom{===}\left .+ \frac{2-q}{2}((\bm{A}\bm{x}_k-\bm{y})_i^2+\epsilon_k)^{\frac{q}{2}}\right) \\
      & \leq \frac{q}{2}\norm{\bm{W}_k\circ (\bm{A}\bm{x}-\bm{y})}_2^2 + \frac{2-q}{2}\norm{\sqrt{(\bm{A}\bm{x}_k-\bm{y})^2+\bm{\epsilon}_k}}_q^q,
    \end{aligned}
  \end{equation*}
  where the weight matrix $\bm{W}_k:=  ((\bm{A}\bm{x}_k-\bm{y})^2 +\bm{\epsilon}_k)^{\frac{q-2}{4}}$ and the multiplication follows Hadamard notation.
\end{proof}

\section{Image Restoration with Impulse Noise Level $r_{sp}=0.7$ }
\label{sec:ao2}
In this section, we present the results of image restoration under a more challenging setting with higher impulse noise, specifically with a salt-and-pepper noise level of $r_{sp}=0.7$. All hyperparameters and evaluation settings remain the same as in the main paper for the $r_{sp}=0.5$ case.

Table~\ref{tab:salt07} reports the quantitative performance on the FFHQ and ImageNet datasets. Our method consistently achieves the best performance across all tasks, further demonstrating the robustness and effectiveness of combining the plug-and-play diffusion prior with the $\ell_q$-norm data fidelity term for handling severe impulse noise.

Note that some cells in the table are left blank, indicating failure modes where the corresponding methods were unable to produce valid reconstructions. These failure cases highlight potential limitations of existing methods under high noise conditions and warrant further investigation to enhance their robustness.

\begin{table}[!htp]
  \centering
  \small
    \fontsize{9}{12}\selectfont 
  \caption{Quantitative results of different methods for different tasks on the  two datasets with $0.7$ salt-and-pepper noise. \label{tab:salt07}}
 \begin{tabular}{c@{\hspace{3pt}}c@{\hspace{1pt}}c@{\hspace{1pt}}c@{\hspace{1pt}}c@{\hspace{1pt}}c@{\hspace{1pt}}c@{\hspace{1pt}}c@{\hspace{1pt}}c@{\hspace{1pt}}c@{\hspace{1pt}}c@{\hspace{1pt}}c@{\hspace{1pt}}c@{\hspace{1pt}}c@{\hspace{1pt}}c@{\hspace{1pt}}c@{\hspace{1pt}}cc@{\hspace{1pt}}c@{\hspace{1pt}}c@{\hspace{0pt}}cc@{\hspace{1pt}}c@{\hspace{1pt}}c@{\hspace{0pt}}cc@{\hspace{1pt}}c@{\hspace{1pt}}c@{\hspace{0pt}}}
   \toprule
 \multirow{2}{*}{\STAB{\rotatebox[origin=c]{90}{Data}}} & \multirow{2}{*}{Method}&   \multicolumn{2}{c}{\textbf{Denoise}}  && \multicolumn{2}{c}{\textbf{Deblur}} && \multicolumn{2}{c}{\textbf{Inpaint}}  && \multicolumn{2}{c}{\textbf{SR}} \\
        \cline{3-4} \cline{6-7} \cline{9-10} \cline{12-13}
   &   & PSNR  & SSIM && PSNR  & SSIM  && PSNR  & SSIM && PSNR  & SSIM \\
   \toprule
 \multirow{4}{*}{\STAB{\rotatebox[origin=c]{90}{FFHQ}}} & PnP-TV      & 18.16 & 0.475  &  & 19.92 & 0.562  &  & 19.64 & 0.606  &  & 16.29 & 0.519  \\
& DRUnet & 18.97 & 0.5581 &  & 23.17 & 0.6756 &  & 12.26 & 0.4793 &  & 12.91 & 0.4637 \\
& DPS       & 13.92 & 0.434 &  & 13.85 & 0.486 &  & 10.26 & 0.375 &  & 9.58 & 0.038\\
& Ours     & 23.84 & 0.713  &  & 23.48 & 0.691  &  & 20.69 & 0.637  &  & 20.24 & 0.621 \\
\midrule
\multirow{4}{*}{\STAB{\rotatebox[origin=c]{90}{ImageNet}}} & PnP-TV     & 14.08 & 0.223  &  & 17.85 & 0.467 &  & 18.74 & 0.533 &  & 16.35 & 0.4405 \\
& DRUnet & 17.12 & 0.435 &  & 18.33 & 0.399  &  & 12.44 & 0.429 &  &       &        \\
& DPS         & 13.64 & 0.272  &  & 12.05 & 0.242 &  & 7.61  & 0.232 &  & 8.43  & 0.026  \\
   & Ours        & 21.65 & 0.6149   &  & 21.88 & 0.577 &  & 20.39 & 0.534 &  &       &       \\
   \bottomrule
\end{tabular}
\end{table}

\begin{figure*}[!ht]
    \centering 
    \begin{tabular}{c@{\hspace*{3pt}}c@{\hspace*{1pt}}c@{\hspace*{1pt}}c@{\hspace*{1pt}}c@{\hspace*{1pt}}c@{\hspace*{1pt}}c@{\hspace*{1pt}}c@{\hspace*{1pt}}c@{\hspace*{1pt}}c@{\hspace*{2pt}}c@{\hspace*{2pt}}c@{\hspace*{2pt}}c@{\hspace*{1pt}}}
\rotatebox[origin=c]{90}{Denoising} &\raisebox{-0.5\height}{\includegraphics[width = 0.15\textwidth]{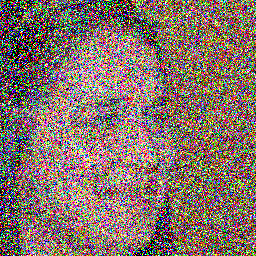}}
  & \raisebox{-0.5\height}{\includegraphics[width = 0.15\textwidth]{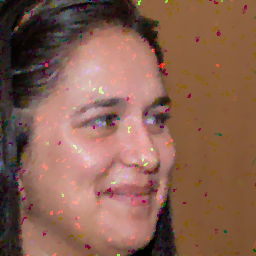}}
  & \raisebox{-0.5\height}{\includegraphics[width = 0.15\textwidth]{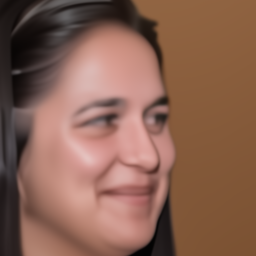}}
  & \raisebox{-0.5\height}{\includegraphics[width = 0.15\textwidth]{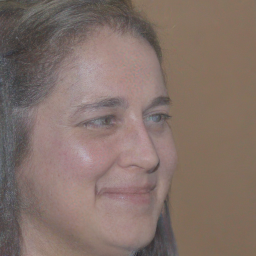}}
  & \raisebox{-0.5\height}{\includegraphics[width = 0.15\textwidth]{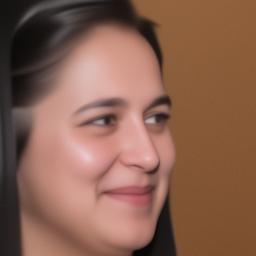}}
  & \raisebox{-0.5\height}{\includegraphics[width = 0.15\textwidth]{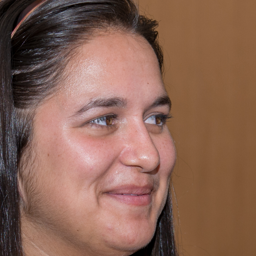}}\\[-2pt]
  \rotatebox[origin=c]{90}{Denoising} &\raisebox{-0.5\height}{\includegraphics[width = 0.15\textwidth]{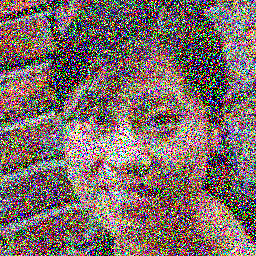}}
  & \raisebox{-0.5\height}{\includegraphics[width = 0.15\textwidth]{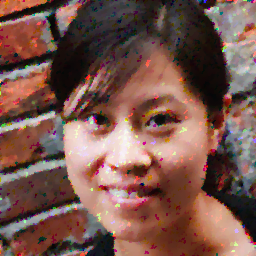}}
  & \raisebox{-0.5\height}{\includegraphics[width = 0.15\textwidth]{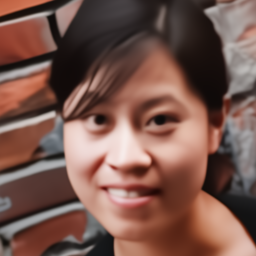}}
  & \raisebox{-0.5\height}{\includegraphics[width = 0.15\textwidth]{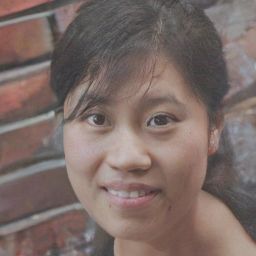}}
  & \raisebox{-0.5\height}{\includegraphics[width = 0.15\textwidth]{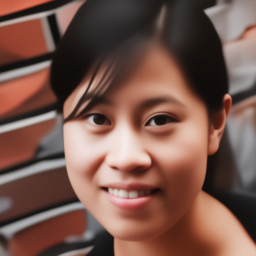}}
      & \raisebox{-0.5\height}{\includegraphics[width = 0.15\textwidth]{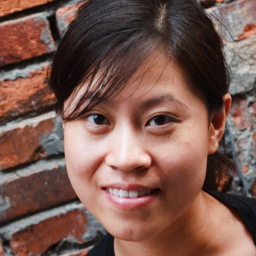}}\\[-2pt]
   \rotatebox[origin=c]{90}{Deblur} &\raisebox{-0.5\height}{\includegraphics[width = 0.15\textwidth]{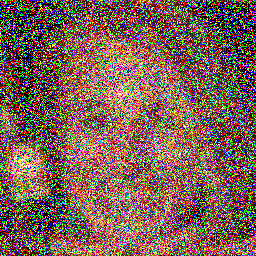}}
& \raisebox{-0.5\height}{\includegraphics[width = 0.15\textwidth]{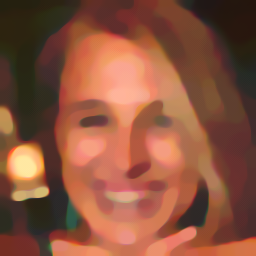}}
& \raisebox{-0.5\height}{\includegraphics[width = 0.15\textwidth]{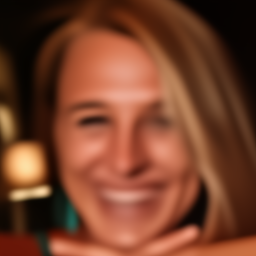}}
& \raisebox{-0.5\height}{\includegraphics[width = 0.15\textwidth]{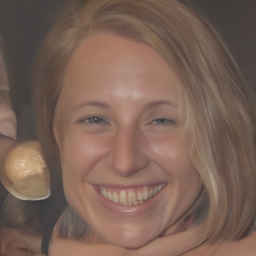}}
& \raisebox{-0.5\height}{\includegraphics[width = 0.15\textwidth]{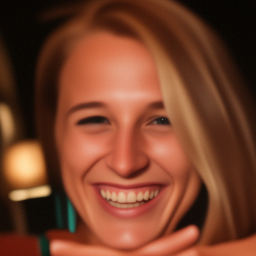}}
& \raisebox{-0.5\height}{\includegraphics[width = 0.15\textwidth]{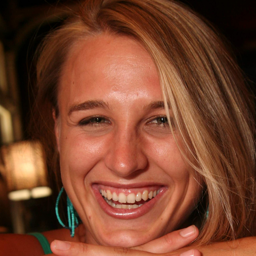}}\\[-2pt]
\rotatebox[origin=c]{90}{Deblur} &\raisebox{-0.5\height}{\includegraphics[width = 0.15\textwidth]{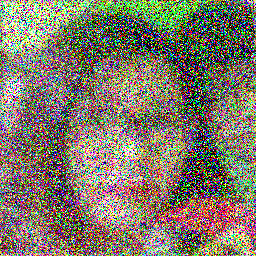}}
& \raisebox{-0.5\height}{\includegraphics[width = 0.15\textwidth]{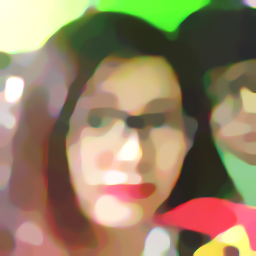}}
& \raisebox{-0.5\height}{\includegraphics[width = 0.15\textwidth]{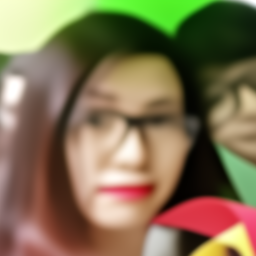}}
& \raisebox{-0.5\height}{\includegraphics[width = 0.15\textwidth]{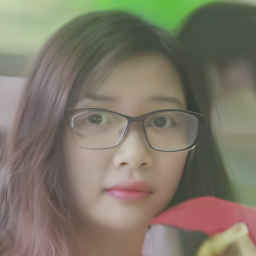}}
& \raisebox{-0.5\height}{\includegraphics[width = 0.15\textwidth]{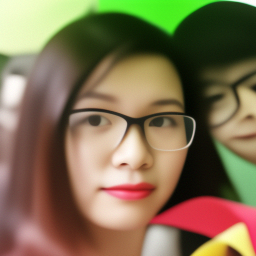}}
& \raisebox{-0.5\height}{\includegraphics[width = 0.15\textwidth]{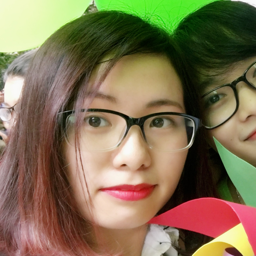}}\\[-2pt]
\rotatebox[origin=c]{90}{Inpaint} &\raisebox{-0.5\height}{\includegraphics[width = 0.15\textwidth]{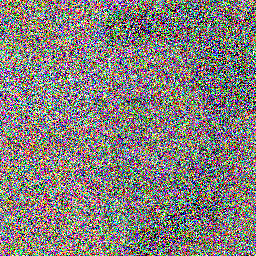}}
& \raisebox{-0.5\height}{\includegraphics[width = 0.15\textwidth]{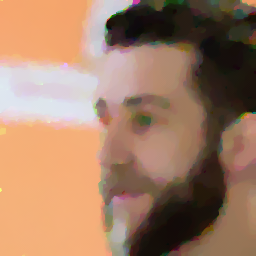}}
& \raisebox{-0.5\height}{\includegraphics[width = 0.15\textwidth]{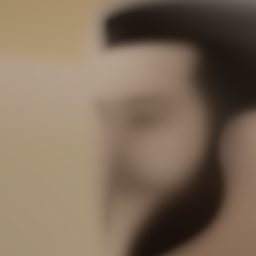}}
& \raisebox{-0.5\height}{\includegraphics[width = 0.15\textwidth]{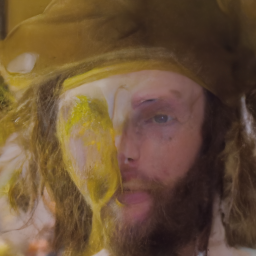}}
& \raisebox{-0.5\height}{\includegraphics[width = 0.15\textwidth]{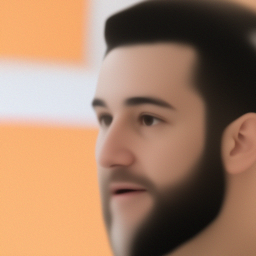}}
& \raisebox{-0.5\height}{\includegraphics[width = 0.15\textwidth]{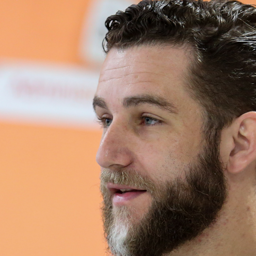}}\\[-2pt]
\rotatebox[origin=c]{90}{Inpaint} &\raisebox{-0.5\height}{\includegraphics[width = 0.15\textwidth]{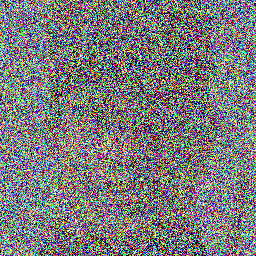}}
& \raisebox{-0.5\height}{\includegraphics[width = 0.15\textwidth]{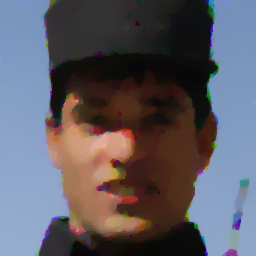}}
& \raisebox{-0.5\height}{\includegraphics[width = 0.15\textwidth]{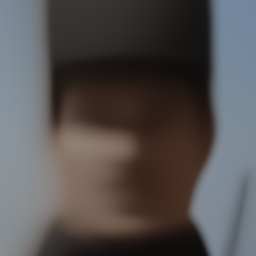}}
& \raisebox{-0.5\height}{\includegraphics[width = 0.15\textwidth]{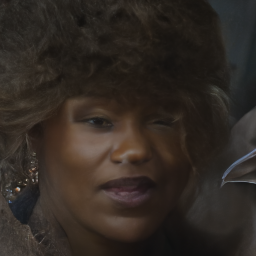}}
& \raisebox{-0.5\height}{\includegraphics[width = 0.15\textwidth]{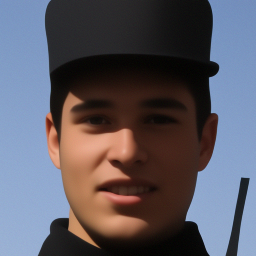}}
      & \raisebox{-0.5\height}{\includegraphics[width = 0.15\textwidth]{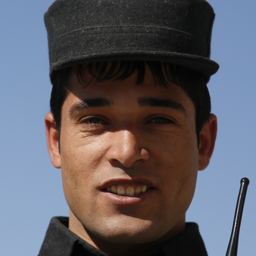}}\\[-2pt]
\rotatebox[origin=c]{90}{SR} &\raisebox{-0.5\height}{\includegraphics[width = 0.15\textwidth]{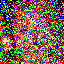}}
& \raisebox{-0.5\height}{\includegraphics[width = 0.15\textwidth]{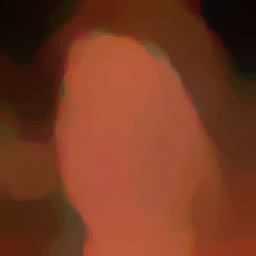}}
& \raisebox{-0.5\height}{\includegraphics[width = 0.15\textwidth]{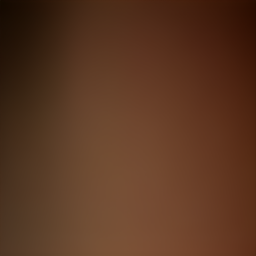}}
& \raisebox{-0.5\height}{\includegraphics[width = 0.15\textwidth]{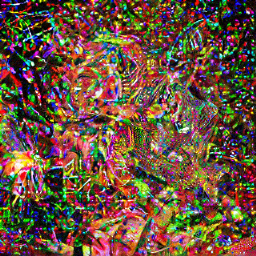}}
& \raisebox{-0.5\height}{\includegraphics[width = 0.15\textwidth]{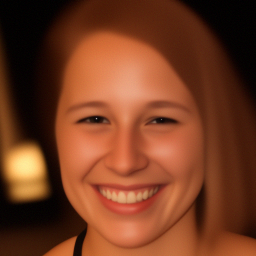}}
& \raisebox{-0.5\height}{\includegraphics[width = 0.15\textwidth]{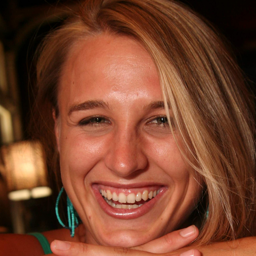}}\\[-2pt]
\rotatebox[origin=c]{90}{SR} &\raisebox{-0.5\height}{\includegraphics[width = 0.15\textwidth]{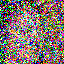}}
& \raisebox{-0.5\height}{\includegraphics[width = 0.15\textwidth]{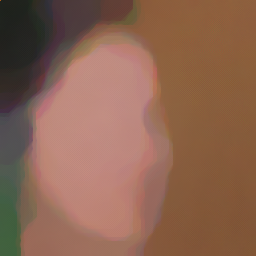}}
& \raisebox{-0.5\height}{\includegraphics[width = 0.15\textwidth]{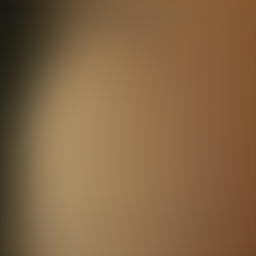}}
& \raisebox{-0.5\height}{\includegraphics[width = 0.15\textwidth]{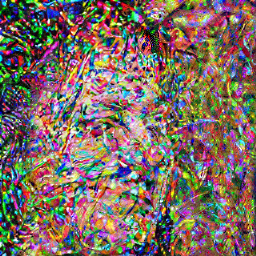}}
& \raisebox{-0.5\height}{\includegraphics[width = 0.15\textwidth]{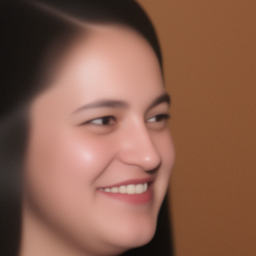}}
& \raisebox{-0.5\height}{\includegraphics[width = 0.15\textwidth]{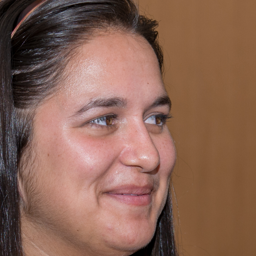}}\\
      &  \small{Input} & \small{PnP-TV} & \small{DRUnet} & \small{DPS} & \small{Ours} & \small{GT}\\
  \end{tabular}\vspace*{-8pt}
  \caption{\textbf{Visualization of different methods on four tasks with impulse noise for FFHQ.}\label{fig:ffhq_a} } 
  \end{figure*} 

\section{More Visualization Comparisons on FFHQ and AFHQ}\label{sec:a1}

In this section, we provide additional visual comparisons to further demonstrate the effectiveness of our proposed method. Figures~\ref{fig:ffhq_a} and~\ref{fig:afhq_a} show restoration results on the FFHQ and AFHQ datasets, respectively, under impulse noise level $r_{sp}=0.5$.

Our method consistently produces high-quality reconstructions, significantly outperforming alternative plug-and-play (PnP) priors, including traditional total variation regularization and PnP methods with pretrained denoisers. Notably, our approach achieves better detail preservation and fewer artifacts, particularly in challenging regions affected by impulse noise, illustrating the superiority of integrating a diffusion prior with $\ell_q$-norm fidelity for robust image restoration.

 \begin{figure*}[!htbp]
    \centering 
    \begin{tabular}{c@{\hspace*{3pt}}c@{\hspace*{1pt}}c@{\hspace*{1pt}}c@{\hspace*{1pt}}c@{\hspace*{1pt}}c@{\hspace*{1pt}}c@{\hspace*{1pt}}c@{\hspace*{1pt}}c@{\hspace*{1pt}}c@{\hspace*{2pt}}c@{\hspace*{2pt}}c@{\hspace*{2pt}}c@{\hspace*{1pt}}}
 \rotatebox[origin=c]{90}{Denoising} &\raisebox{-0.5\height}{\includegraphics[width = 0.15\textwidth]{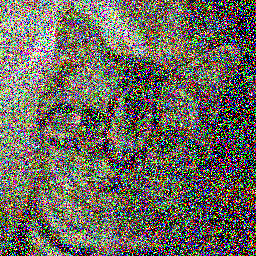}}
  & \raisebox{-0.5\height}{\includegraphics[width = 0.15\textwidth]{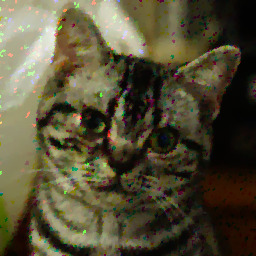}}
  & \raisebox{-0.5\height}{\includegraphics[width = 0.15\textwidth]{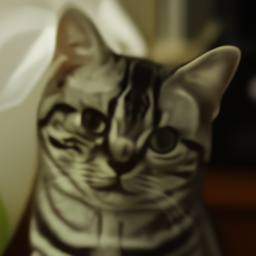}}
  & \raisebox{-0.5\height}{\includegraphics[width = 0.15\textwidth]{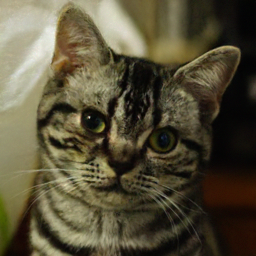}}
  & \raisebox{-0.5\height}{\includegraphics[width = 0.15\textwidth]{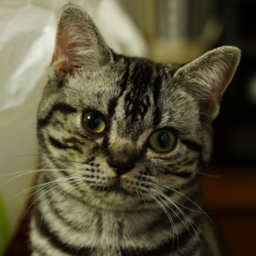}}\\[-2pt]
 
  \rotatebox[origin=c]{90}{Denoising} &\raisebox{-0.5\height}{\includegraphics[width = 0.15\textwidth]{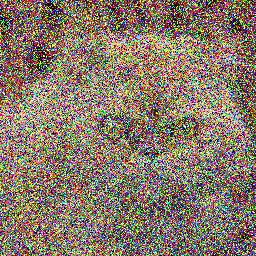}}
  & \raisebox{-0.5\height}{\includegraphics[width = 0.15\textwidth]{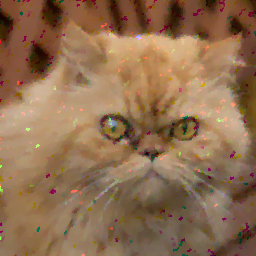}}
  & \raisebox{-0.5\height}{\includegraphics[width = 0.15\textwidth]{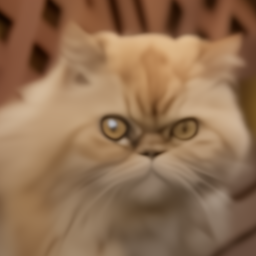}}
  & \raisebox{-0.5\height}{\includegraphics[width = 0.15\textwidth]{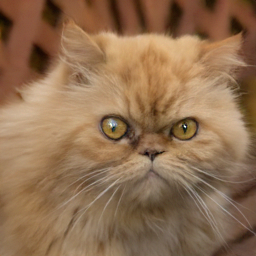}}
      & \raisebox{-0.5\height}{\includegraphics[width = 0.15\textwidth]{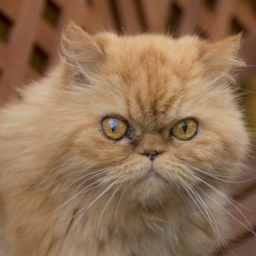}}\\[-2pt]

      \rotatebox[origin=c]{90}{Deblur} &\raisebox{-0.5\height}{\includegraphics[width = 0.15\textwidth]{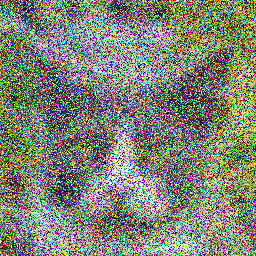}}
  & \raisebox{-0.5\height}{\includegraphics[width = 0.15\textwidth]{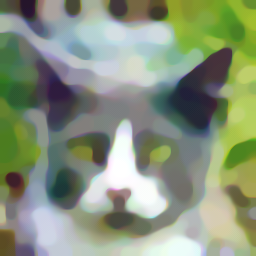}}
  & \raisebox{-0.5\height}{\includegraphics[width = 0.15\textwidth]{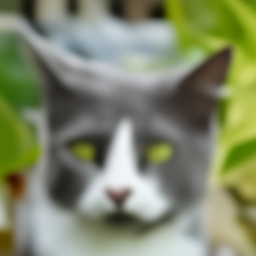}}
  & \raisebox{-0.5\height}{\includegraphics[width = 0.15\textwidth]{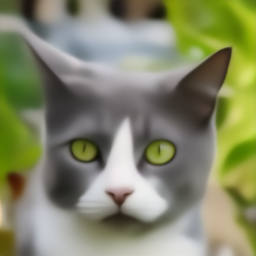}}
  & \raisebox{-0.5\height}{\includegraphics[width = 0.15\textwidth]{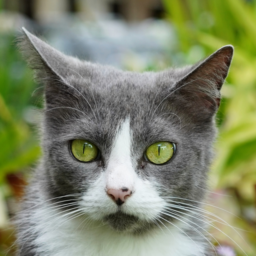}}\\[-2pt]
  \rotatebox[origin=c]{90}{Deblur} &\raisebox{-0.5\height}{\includegraphics[width = 0.15\textwidth]{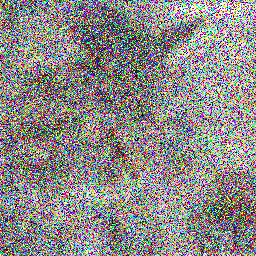}}
  & \raisebox{-0.5\height}{\includegraphics[width = 0.15\textwidth]{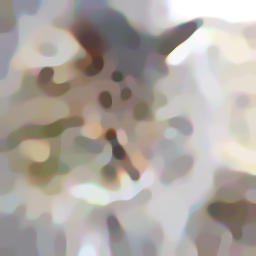}}
  & \raisebox{-0.5\height}{\includegraphics[width = 0.15\textwidth]{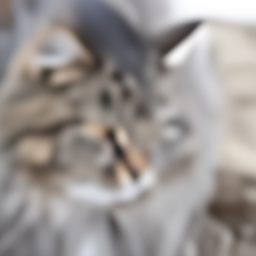}}
  & \raisebox{-0.5\height}{\includegraphics[width = 0.15\textwidth]{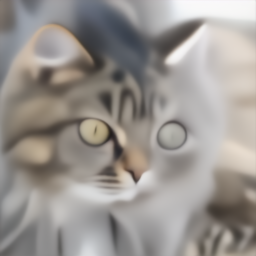}}
      & \raisebox{-0.5\height}{\includegraphics[width = 0.15\textwidth]{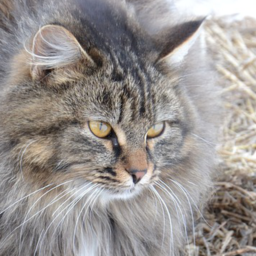}}\\[-2pt]
      \rotatebox[origin=c]{90}{Inpaint} &\raisebox{-0.5\height}{\includegraphics[width = 0.15\textwidth]{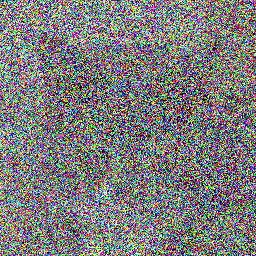}}
  & \raisebox{-0.5\height}{\includegraphics[width = 0.15\textwidth]{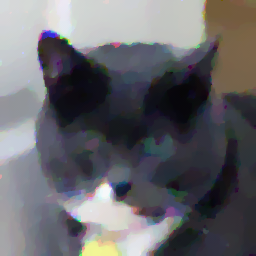}}
  & \raisebox{-0.5\height}{\includegraphics[width = 0.15\textwidth]{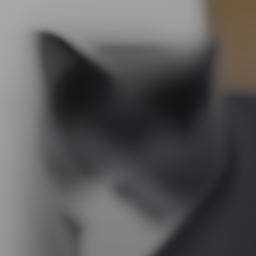}}
  & \raisebox{-0.5\height}{\includegraphics[width = 0.15\textwidth]{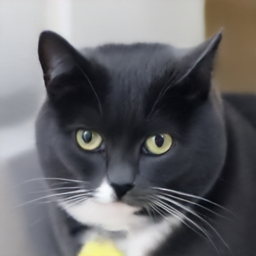}}
  & \raisebox{-0.5\height}{\includegraphics[width = 0.15\textwidth]{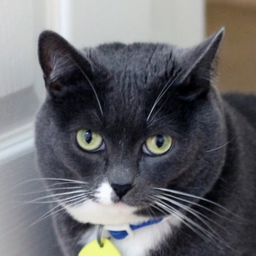}}\\[-2pt]
  \rotatebox[origin=c]{90}{Inpaint} &\raisebox{-0.5\height}{\includegraphics[width = 0.15\textwidth]{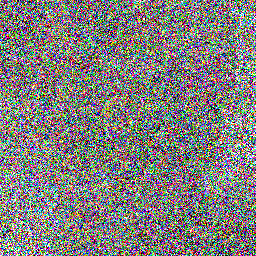}}
  & \raisebox{-0.5\height}{\includegraphics[width = 0.15\textwidth]{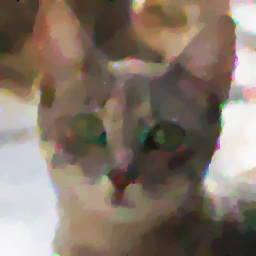}}
  & \raisebox{-0.5\height}{\includegraphics[width = 0.15\textwidth]{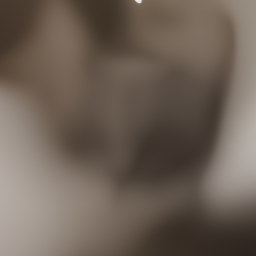}}
  & \raisebox{-0.5\height}{\includegraphics[width = 0.15\textwidth]{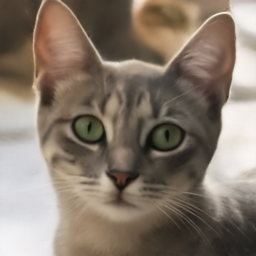}}
      & \raisebox{-0.5\height}{\includegraphics[width = 0.15\textwidth]{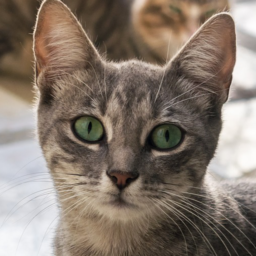}}\\[-2pt]
        \rotatebox[origin=c]{90}{SR} &\raisebox{-0.5\height}{\includegraphics[width = 0.15\textwidth]{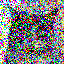}}
  & \raisebox{-0.5\height}{\includegraphics[width = 0.15\textwidth]{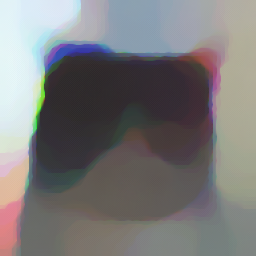}}
  & \raisebox{-0.5\height}{\includegraphics[width = 0.15\textwidth]{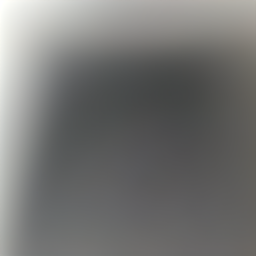}}
  & \raisebox{-0.5\height}{\includegraphics[width = 0.15\textwidth]{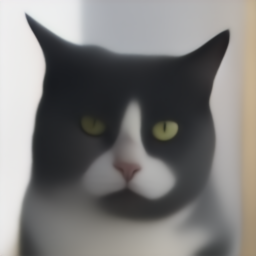}}
  & \raisebox{-0.5\height}{\includegraphics[width = 0.15\textwidth]{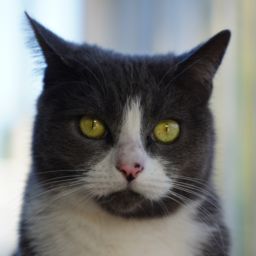}}\\[-2pt]
  \rotatebox[origin=c]{90}{SR} &\raisebox{-0.5\height}{\includegraphics[width = 0.15\textwidth]{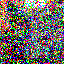}}
  & \raisebox{-0.5\height}{\includegraphics[width = 0.15\textwidth]{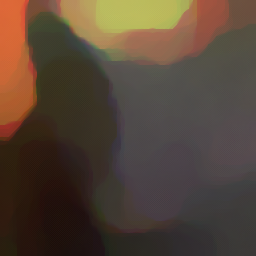}}
  & \raisebox{-0.5\height}{\includegraphics[width = 0.15\textwidth]{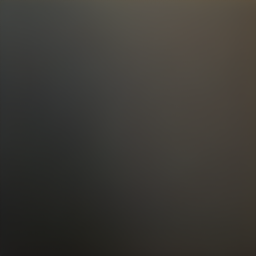}}
  & \raisebox{-0.5\height}{\includegraphics[width = 0.15\textwidth]{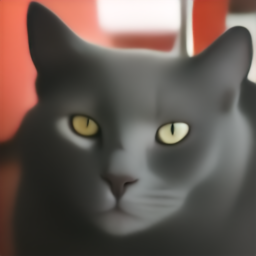}}
  & \raisebox{-0.5\height}{\includegraphics[width = 0.15\textwidth]{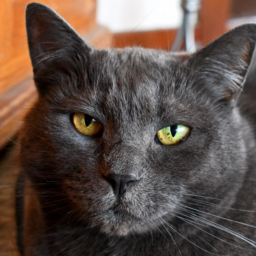}}\\
      &  \small{Input} & \small{PnP-TV} & \small{DRUnet} & \small{Ours} & \small{GT}\\
  \end{tabular}\vspace*{-8pt}
  \caption{\textbf{Visualization of different methods on four tasks with impulse noise for AFHQ-Cat.}\label{fig:afhq_a} } 
  \end{figure*} 
\end{document}